\documentclass[2pt]{article}
\usepackage{fullpage}
\usepackage{amsmath,amssymb,amsthm,float,epstopdf,makeidx,latexsym}
\usepackage{pgfplots,lipsum}
\usepackage{amsmath, amsthm, amscd, amsfonts, amssymb, graphicx, color}
\usepackage{tikz}
\usepackage{booktabs}
\usepackage{authblk}
\usepackage{pgfplots}
\usepackage{graphicx}
\usepackage{pgf,tikz}
\usetikzlibrary{arrows.meta}
\usepackage{blkarray}
\usetikzlibrary{calc, positioning}\usepackage{setspace}
\usepackage{rotating}
\usepackage{algpseudocode}
\usepackage{algorithm}
\usepackage{pdflscape}
\usepackage{tabularx}
\usetikzlibrary{arrows}
\theoremstyle{plain}
\usepackage{amsmath}
\usepackage{graphicx}
\newtheorem{definition}{Definition}
\newtheorem{example}{Example}
\newtheorem{proposition}{Proposition}
\newtheorem{theorem}{Theorem}
\usepackage{mwe}

\title{ Architectural  change in neural networks using fuzzy vertex pooling}

\author[1]{Shanookha Ali\thanks{shanookhaali@gmail.com}}
\author[2]{Nitha Niralda\thanks{nithaniraldapc@gmail.com}}
\author[3]{Sunil Mathew\thanks{sm@nitc.ac.in}}
\affil[1]{\textit{Department of General Science, Birla Institute of Technology and Science Pilani, International Academic City, P.O. Box: 345055,  Dubai}}
\affil[2]{\textit{Department of Mathematics, Providence Women's College, Calicut, India }}
\affil[3]{\textit{Department of Mathematics, National Institute of Technology, Calicut, India }}

\date{}
\begin{document}
	\maketitle

	\begin{abstract}
 The process of pooling vertices involves the creation of a new vertex, which becomes adjacent to all the vertices that were originally adjacent to the endpoints of the vertices being pooled. After this, the endpoints of these vertices and all edges connected to them are removed. In this document, we introduce a formal framework for the concept of fuzzy vertex pooling (FVP) and provide an overview of its key properties with its applications to neural networks. The pooling model demonstrates remarkable efficiency in minimizing loss rapidly while maintaining competitive accuracy, even with fewer hidden layer neurons. However, this advantage diminishes over extended training periods or with larger datasets, where the model's performance tends to degrade. This study highlights the limitations of pooling in later stages of deep learning training, rendering it less effective for prolonged or large-scale applications. Consequently, pooling is recommended as a strategy for early-stage training in advanced deep learning models to leverage its initial efficiency.
\end{abstract}

%\begin{IEEEkeywords}
fuzzy vertex pooling, f-graph pooling,  f-cycle pooling,  neural network, fuzzy neural network.
%\end{IEEEkeywords}

\section{Introduction}
{I}n recent years, applying deep learning to f-graphs is a fast-growing field. The application of Fuzzy Convolutional Neural Networks (FCNNs) to the sparsely linked data that f-graphs depict serves as the basis for many of these research. While numerous Fuzzy Graph Convolutional Networks (FGCNs) have been presented, there haven't been many pooling layer proposals. However, intelligent pooling on fuzzy networks has a lot of potential, by lowering the number of vertices, it may be able to both discover clusters and minimise computing requirements. These two things hold the promise of transforming flat vertices into hierarchical sets of vertices. They are also a first step toward FGNNs being able to alter graph topologies rather than just vertex properties \cite{murphy2019relational}.\medskip

\begin{figure}
	\centering
	\includegraphics[width=9cm, height=8 cm]{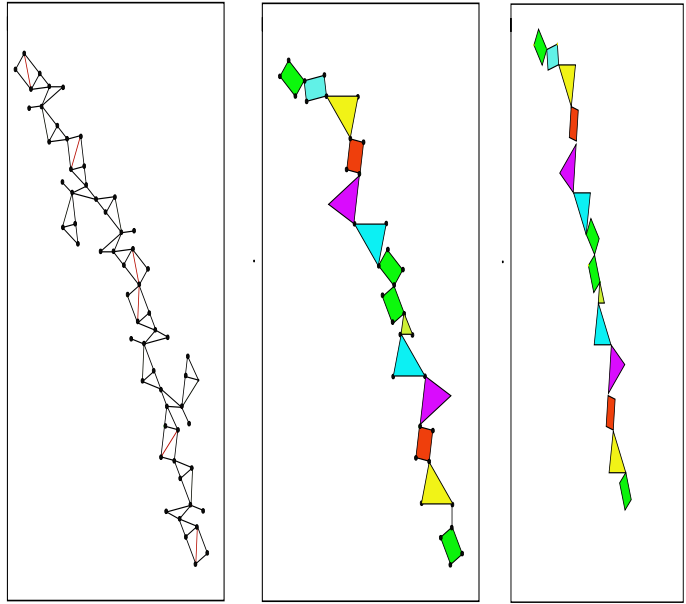}
	\caption{Vertex pooling in a f-graph. The graph to the right is produced by pooling the original f-graph twice.}
\end{figure}

\noindent In 1965, Zadeh \cite{zadeh1965fuzzy} introduced the paradigm-shifting idea of fuzzy logic, which was crucial in influencing how many mathematical and engineering ideas are now seen. Rosenfeld used this logic in\textbf{ 1975} \cite{rosenfeld1975fuzzy} to redefine a number of graph theory terms. These days, fuzzy graph theory is a popular topic in mathematics.  Mathew and Sunitha investigated fuzzy vertex connectivity and fuzzy edge connectivity \cite{mathew2010node} in 2010. In 2018, average fuzzy vertex connection was proposed by Shanookha et al after further research on fuzzy vertex connectivity and fuzzy containers  \cite{ali2018vertex,ali2024containers} later they formulated concepts of hamiltonian fuzzy graphs and applied them in analysis of Human trafficking \cite{ali2021hamiltonian} in 2021. \medskip

\noindent It makes no difference in vertex pooling if two vertices are joined by an edge; if they are, the edge will vanish when they are pooled. The idea of pooling the vertices was adopted by Pemmaraju and Skiena  \cite{skiena1990combinatorics}. Two  vertices $p$ and $q$ are connected in an undirected f-graph if a path can be found between them. Every pair of vertices must be connected for an f-graph to be considered connected. Partitioning an undirected graph into its most connected subgraphs is the goal of the f-graph connectivity problem.  The term \lq\lq f-graph pooling\rq\rq refers to this method. In addition to connection issues, it is a rather straightforward technique that may be used to solve f- cycle and f- tree issues as well. We assume the f-graph is undirected.\medskip

\noindent Fuzzy sets can be used to describe various aspects of Neural
Computing. That is, fuzziness may be introduced at the input output signals, synaptic weights, aggregation operation and activation function of individual neurons to make it fuzzy neuron. Different aggregation operations
and activation functions result in fuzzy neurons with different properties.
Thus there are many possibilities for fuzzification of an artificial neuron.
So we may find a variety of fuzzy neurons in the literature \cite{ibrahim1996introduction,tsoukalas1996fuzzy,rajasekaran2003neural,ross1997fuzzy}.  Sameena and Sunitha \cite{sameena2009fuzzy}, studeied  the fuzzy neural network architecture is isomorphic to the fuzzy graph model and the output of a fuzzy neural
network with OR fuzzy neuron is equal to the strength of strongest
path between the input layer and
the out put layer.  Murphy et. al \cite{murphy2019relational} generalizes graph neural networks (GNNs) beyond those based on the Weisfeiler-Lehman (WL) algorithm, graph Laplacians, and diffusions.\medskip

\noindent The work by Ramya and Lavanya \cite{ramya2023contraction} delves into the specific operations of pooling and domination
in fuzzy graphs. Pooling  involves merging vertices based on their memberships, while
domination deals with identifying dominating vertices. These operations are fundamental in
simplifying graph structures and could have implications for optimizing information flow and
representation learning in GNNs. Qasim et al \cite{qasim2019learning} introduced distance-weighted graph networks for learning representations of irregular
particle detector geometry. While not directly related to fuzzy graphs, their approach highlights
the importance of weighted connections in graph-based learning. This is particularly relevant when
considering how fuzzy graph pooling can influence the weighting of edges and vertices in
GNNs.  \medskip

\noindent In \cite{al2019fuzzy}, Mutab provides a mathematical exploration of fuzzy graphs, discussing their properties and
applications. The paper offers insights into how fuzzy graphs can model imprecise relationships,
an aspect that aligns with the core principle of fuzzy memberships. Understanding the
mathematical underpinnings is vital for integrating fuzzy graph pooling into GNN
architectures. This is
because pooling reduce the complexity of hidden layers and eliminate weak edges within
them. Though this changes the structure of the network midway while training, creating higher
variance in the output resulting in a higher loss overall. However, with the existing algorithms of
forward and backprogpagation, this loss would be quickly minimized.
So far, a form resembling fuzzy graph pooling has been used in convolutional and pooling
networks, but not on simple fuzzy graph networks. There is no conclusive study conducted on
simple fuzzy graph networks to test this hypothesis. \medskip

The paper introduces vertex pooling in fuzzy graphs, a novel pooling method for Graph Neural Networks (GNNs) based on edge pooling.
Unlike traditional pooling methods, FVP  focuses on merging vertices via edges, preserving the graph structure while avoiding vertex loss.
It is highly efficient, operating on sparse graph representations with run-time and memory scaling linearly with the number of edges.
The approach is flexible and seamlessly integrates into various GNN architectures, including GCN, GraphSAGE, and GIN.
EdgePool introduces a new way to handle edge features, enhancing its applicability in diverse graph-based tasks.
It also supports unpooling, enabling graph reconstruction and effective application to vertex classification tasks.
Experiments show consistent performance improvements when FVP is used, with notable gains in scalability for larger graphs.
The method is particularly suitable for localized changes that avoid computational overhead for the entire graph.
In general, FVP is a significant advancement in fuzzy graph pooling, offering a practical, scalable, and robust solution for modern GNN applications. In addition, in this paper we show that the FGP problem is NP-hard depends on the characterization of pooling method.
\section{ Fundamental concepts}
Most of the definitions and preliminary results used in this paper can be seen in \cite{bhattacharya1987some,bhutani2003strong,mathew2018fuzzy,mathew2009types,mordeson2012fuzzy}. Let $\zeta$ be a set. The triplet  $G=(\zeta,\sigma,\mu) $ is called a \emph{f-graph} if each element $\zeta \in \zeta$ and  $pq \in \zeta \times \zeta$ are assigned  non-negative real numbers $\sigma(v)$ and $\mu(pq),$ where $\sigma : \zeta \rightarrow [0, 1] $ and $\mu : \zeta \times \zeta \rightarrow [0, 1]$, called the membership values of $v$ and $e$ respectively, such that for all $p, q \in \zeta$, $\mu(p q) \leq \sigma(p) \wedge \sigma(q),$ where $\wedge$ denote the minimum. We also denote $G=(\sigma,\mu)$ to represent a f-graph. All f-graphs considered in this article are simple and connected.  $\sigma^*$  and  $\mu^*$  respectively denote the   \textit{vertex set} and \textit{edge set} of the f-graph $G$.\medskip

\noindent A f-graph $H = (\tau,\nu)$ is called a \emph{partial f-subgraph} of $G = (\sigma, \mu)$ if $\tau (v)\leq \sigma (v)$ for all $v\in \tau^*$ and $\nu(uv) \leq	 \mu(uv)$ for all $uv \in \nu^*$. If $H = (\tau,\nu)$ is a partial f-subgraph of $G = (\sigma, \mu)$ such that $\tau(v) = \sigma(v)$ for all $v \in \tau^*$ and	 $\nu(uv) = \mu(uv)$ for all $uv \in \nu^*$, then $H$ is called a \emph{f-subgraph} of $G$. $G-pq$ ($G-v$)  is a f-subgraph of $G$ obtained by deleting the edge $pq$ ( or vertex $v$)  from $G$.  Note that a f-subgraph $H = (\tau,\nu)$ spans the f-graph $G = (\sigma, \mu)$ if $\tau = \sigma$.	  A \emph{path} $P$ in a f-graph $G=( \sigma, \mu)$ is a sequence of distinct vertices $p_{0}, p_{1}, \cdots, p_{n}$ such that $\mu(p_{i-1} p_{i})>0$, $1 \leq i \leq n$ and all the vertices are distinct except possibly the first and last. The \emph{strength} of a path, $s(P)$ is the membership value of the weakest edge of the path and it always lies between 0 and 1. If $s(P) = CONN_G (p, q)$, then $P$ is called a \emph{strongest} $p-q$ path. $G$ is said to be \emph{connected} if $CONN_G (p, q) > 0$ for each $p, q \in \sigma^*$. 	  An edge $pq$ of a f-graph $G$ is called $\alpha$-\emph{strong} if $\mu(pq) > CONN_{G-pq} (p, q)$.  $pq \in \mu^*$ is called $\beta$-\emph{strong} if $\mu(pq) = CONN_{G-pq} (p, q)$ and  a $\delta$-\emph{edge} if $\mu(pq) < CONN_{G-pq} (p, q)$. Thus an edge is \emph{strong} if $\mu(pq) \geq CONN_{G-pq} (p, q)$. A path $P$ is called a \emph{strong path} if all of its edges are strong. If the removal of an edge ( or a vertex) reduces the strength of connectedness between some pair of vertices in $G$, then the edge ( vertex ) is called a \emph{fuzzy bridge} ( or \emph{fuzzy cutvertex} ) of $G$.\medskip

\noindent A f-graph without fuzzy cutvertices is called a \emph{f- block / block}.  A connected f-graph $G = (\sigma, \mu)$ is called a \emph{f- tree} if it has a spanning f-subgraph $F = (\sigma, \nu)$ which is a tree, where for all $pq$ not in $\nu^*$, there exists a path from $p$ to $q$ in $F$, whose strength is more than $\mu(pq)$.
The \textit{degree} of a vertex $p\in \sigma^{\ast}$ is defined as $d(p) = \sum_{q \neq p} \mu(qp).$ The \textit{minimum degree} of $G$ is $\delta(G)= \wedge \{ d(p) : p \in \sigma^*\}$
and the maximum degree of $G$ is $\Delta(G) = \vee \{ d(p) = p \in \sigma^*\}.$ The \textit{strong degree} of a vertex $p \in \sigma^*$ is defined as the sum of membership
values of all strong edges incident at $p.$ It is denoted by $d_s(p,G).$ Also if $N_s(p,G)$ denote the set of all strong neighbors of $p$, then  $d_s(p,G) = \sum_{q \in N_s(p,G)} \mu(qp).$\medskip

\noindent Here are few notations that is carried out in this paper:
\begin{eqnarray*}
	d_{s_{G}}(p) &-& \text{ Strong degree of vertex $p$ in $G$}\\
	N_{s_{G}}(p)&-& \text{ Strong neighbor set of vertex $p$ in $G$}\\
\end{eqnarray*}
The pooling of two vertices in a graph, $p_i$ and $p_j$, results in a graph where the two original vertices, $p_i$ and $p_j$, are replaced by a single vertex $p$, which is next to the union of the original neighboring vertices. If $p_i$ and $p_j$ are joined by an edge during vertex pooling, it makes no difference because the edge will vanish after $p_i$ and $p_j$ are pooled. 

\section{Fuzzy vertex pooling in f-graphs}

It is easy to intuitively understand the meaning of  fuzzy vertex pooling (FGP). In Figure \ref{fig2}, we can see
an example when we pool adjacent vertices $g$ and $j$ that belongs to a  f- cycle $g-j-k-g$ of strength $0.6$, and after vertex  pooling the result is an edge ( cycle is replaced by a single edge), since the vertices to be pooled might have
common neighbors as  in Figure \ref{fig2}.  This ways of pooling vertices in a f-graph $G$ are  denoted as $G/_{gj}$. The formal definition of fuzzy vertex pooling (FGP) in a f-graphs is as follows.
\begin{definition}\label{def1}
	Let $G: (\sigma, \mu)$ be a f-graph  with $p,q$ $ \in \sigma^*$. Identify $p$ with $q$ as a new vertex $v_c$ , we call this fuzzyfied operation pooling of vertices and the new graph is denoted by $G/_{pq}:  (\sigma_c, \mu_c)$. Formally, $G/_{pq}$ is a f-graph such that 
	\begin{itemize}
		\item[i] $\sigma_c^* = \{\sigma^* \cup v_c\} \setminus \{p,q\}$
		\item[ii] $\mu_c^* = \{\mu^* \setminus \{\{up, vq: u \in N(p) \& v \in N(q) \}  \cup \{ wv_c: w \in (N(p)\setminus\{q \}) \cap  ( N(q)\setminus\{p \}) \}$ with
	\end{itemize} 
	$\sigma_c(v) =\begin{cases}
		\sigma(v) & :v \in \sigma^* \setminus \{p,q\}\\
		\sigma(p)\wedge \sigma(q)& :v=v_c
	\end{cases} $ and 
	
	$\mu_c(uv) =\begin{cases}
		\mu(uv) & :u , v \not\in N(p)\cup N(q)\\
		\mu(uv_c)& :u \not\in N(q)\\
		\mu(vv_c)& :v \not\in N(p) \\
		\mu(up)\wedge \mu(uq)& : u \in N(p) \cap N(q)
	\end{cases} $
\end{definition}
\begin{figure}
	\centering
	\includegraphics[width=8cm, height=5 cm]{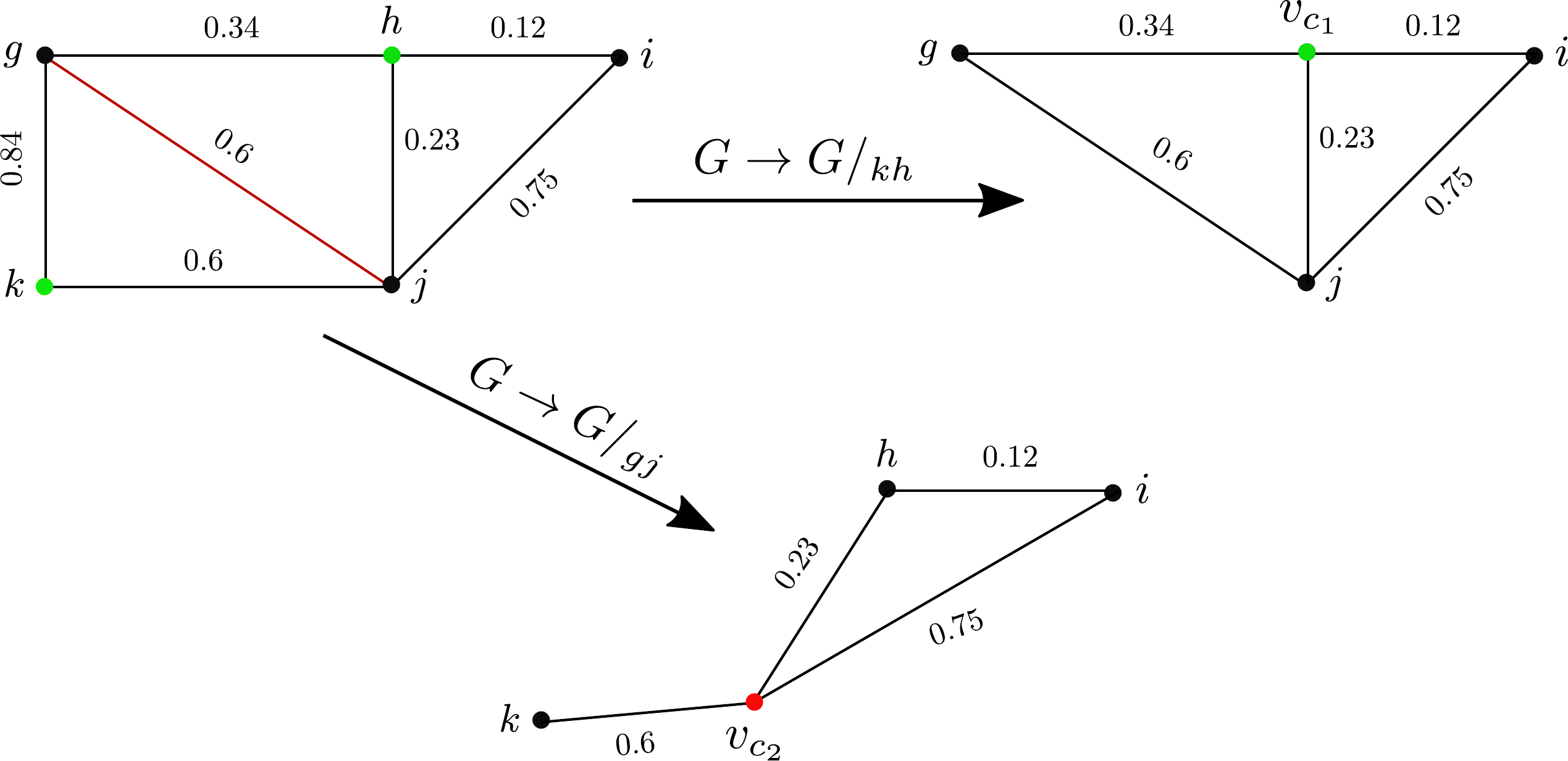}
	\caption{ Original f-graph $G$; resulting f-graph $G/_{kh}$ after fuzzy vertex pooling edge $k\text{ and }h$; resulting f-graph $G/_{gj}$ after fuzzy vertex pooling edge $gj$ }
	\label{fig2}
\end{figure}
\begin{example}
	Let $G:(\sigma, \mu)$ be a f-graph with $\sigma^* = \{g,h,i,j,k\}$ in Figure \ref{fig2}. Let $\{k,h\}$ be the  vertices to pool. Note that the remaining vertices, $\{g,j,i\}$ $\in$ $\sigma^*$ are left on their own left on their own. We call new vertex as $v_{c_1}$. $G/_{kh}:(\sigma_1, \mu_1)$ is the vertex pooled f-subgraph of $G$ with $\sigma_1^* = \{g,i,j,v_{c_1}\}$ and $\mu_{2}^* = \{gj,ji,jv_{c_1},gv_{c_1},v_{c_1}i\}$. Here 
	\begin{eqnarray*}
		\mu(jv_{c_1}) &=& \wedge\{\mu(jh), \mu(jk)\}\\
		\mu(gv_{c_1}) &=& \wedge\{\mu(gh), \mu(gk)\}\\
		\mu(iv_{c_1}) &=& \mu(hi)
	\end{eqnarray*}
	Let $\{g,j\}$ be the  vertices to pool. Note that the remaining vertices, $\{k,h,i\}$ $\in$ $\sigma^*$ are left on their own left on their own. We call new vertex as $v_{c_2}$. $G/_{gj}:(\sigma_2, \mu_2)$ is the vertex pooled f- subgraph of $G$ with $\sigma_2^* = \{k,i,h,v_{c_2}\}$ and $\mu_{2}^* = \{kv_{c_2},hv_{c_2},v_{c_2}i\}, hi$. Here 
	\begin{eqnarray*}
		\mu(kv_{c_2}) &=& \wedge\{\mu(jk), \mu(gk)\}\\
		\mu(hv_{c_2}) &=& \wedge\{\mu(gh), \mu(hj)\}\\
		\mu(iv_{c_1}) &=& \mu(ji)
	\end{eqnarray*}
\end{example}
\begin{proposition}
	Fuzzy vertex pooling in a  f-graph is commutative.
\end{proposition}\label{prop}
\noindent Hence, it is sensible to define the FGP of a set of vertices. Hereafter, we use the following shorthand:
\begin{definition}\label{def1} 
	Given a f-graph $G: (\sigma, \mu)$ and a path $P = \{v_1-v_2-v_3-\cdots- v_p\}$, we define:
	$$ G/_P = G/_{v_1v_2}/_{v_2v_3}/\cdots /_{v_{p-1}v_{p}}$$
	by naming $v_{i+1}$ as the new vertex obtained by pooling $v_i$ and $v_{i+1}$.
\end{definition}
\noindent The above concept can be generalized for any set of vertices.\medskip

\noindent The vertices to be pooled may change as a result of previous pooling when a collection of vertices is being pooled. The following lemma makes a claim about the vertex that is pooled last if a f-cycle is formed in the f-graph by a group of vertex pooling.
\begin{theorem}
	Let $\mathbf{C} = c_1- c_2- \cdots - c_k-c_1 $ be a f-cycle of length $k\geq 4$ with strength $\phi$ and let $E$ be the set of weakest edges in $\mathbf{C}$, 	w.l.o.g. $E = \{(e_1 = c_1 c_2) ,\cdots, (e_r = c_r c_{r+1} )\}$, i.e. $\mu(e_i) = \phi $. Then,	
	\begin{enumerate}
		\item[i] Pool the other vertex combinations first until we get a f- cycle $\mathbf{C'}$ of edges $\{e_1,e_2,\cdots e_r, e_{r+1}\}$, where $\mu(e_{r+1} ) \geq \phi$.
		\item[ii] Pool vertices of $\mathbf{C'}$ to obtain a single vertex, say $c_c$ of $\sigma(c_c) = \phi$ .
	\end{enumerate} 
\end{theorem}
\begin{proof}
	We prove this by induction on $k$. For $k = 4$, we have $\mathbf{C}  = {c_1 -c_2- c_3-c_4-c_1}$. Without loss of generality, $E =\{ (e_1 = {c_1 c_2 }),( e_2 = {c_2 c_3 })\}$, $\mu(e_1) = \mu(e_2) = \phi$. Let $e_3 = c_3 c_4$ and $e_4 = c_4 c_1$ be the other vertices with $\mu(e_3),$ $ \mu(e_4) > \phi.$\medskip
	
	\noindent Let $C_1 = C/_{c_3c_4}$ is a cycle with length 3 and we name new vertex as $c_3$ and new edge $c_3c_1$ with $\mu(c_3c_1) = \wedge \{\mu(c_2 c_3), \mu(c_4 c_1)\} \geq \phi $. Let $C_2 = C/_{c_3c_1}$ is an edge and we name new vertex as $c_1$ and new edge $c_2c_1$ with $\mu(c_2c_1) = \wedge \{\mu(c_1 c_3), \mu(c_2 c_3)\} = \phi$. Pooling this edge $c_1 c_2$ is a single vertex $C_3 = C/_{c_1c_2}$ with membership value $\phi$. 
	
	\begin{figure}
		\centering
		\includegraphics[width=8cm, height=1.8 cm]{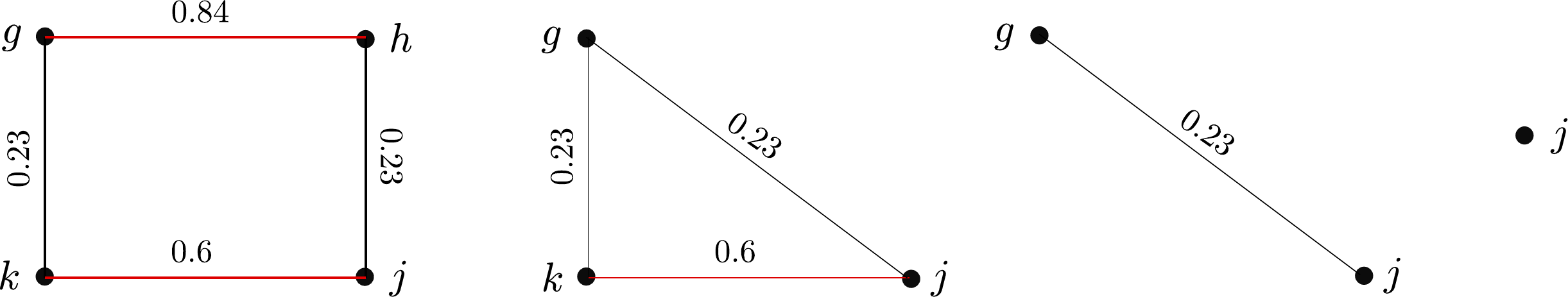}
		\caption{ FGP of a f- cycle $C$ with 4 vertices to $C/_{gh}$, then to $C/_{gh}/_{kj}$ to $C/_{gh}/_{kj}/_{gj}$ to a vertex $j$ }
		\label{cyc}
	\end{figure}
	\noindent Hence, we assume the lemma holds for f- cycles of length $k$, and we will show it also holds for f- cycles of length $k + 1$. Let be given a f- cycle $\mathbf{C}$- $c_1- c_2- \cdots - c_k-c_{k+1}-c_1 $ of length $k + 1$. $C/_{c_kc_{k+1}}$, is a f- cycle obtained by Pooling $c_{k} c_{k+1}$ introduces a new
	vertex a $c_k$ and edge $c_{k}c_{1}$.  Hence, this results in a f- cycle $C/_{c_kc_{k+1}}$ of length $k$, for which we know that the lemma holds ( Figure \ref{cyc} is an illustration).
 %\end{theorem}
\end{proof}
\begin{theorem} 
	Let  $G: (\sigma, \mu)$ be given a f-graph and a set of edges $\mu_1= \{f_1 ,\cdots, f_q \}$ and a set of edges $\mu_2 = \{g_1 ,\cdots, g_p \}$, $\mu_1 \cap \mu_2 = \emptyset$. Let $\mu_2$ form a f-cycle $g_1-\cdots- g_p$ with $p\geq 3$. Let be $\mu_3 = \{f_1 ,\cdots, f_q,g_1 ,\cdots, g_p\}$ and $\mu_3 = \{g_1 ,\cdots, g_p,f_1 ,\cdots, f_q\}$, then $$G/_{\mu_3} \text{ isomorphic to  } G/_{\mu_4}.$$
\end{theorem}
\begin{proof}
	Let, $f_i = v_iv_{i+1}$ and  $g_i = u_iu_{i+1}.$ From Proposition \ref{prop}, we know FGP are commutative. In light of this, we opt to pool the vertices in the following order: $v_1 ,v_2, \cdots v_q,v_{q+1},u_1,u_2,\cdots ,u_p,u_{p+1}$. Then we have by definition %\ref{def1}:
	\begin{eqnarray*}
		G/_{\mu_3}&=& G/_{v_1v_2}/_{v_2v_3}/ \cdots/_{v_qv_{q+1}}/_{v_{q+1}u_1}{u_1u_2}/\cdots /_{u_pu_{p+1}} \\
		&=& G/_{v_1v_2}/_{v_2v_3}/ \cdots/_{v_qv_{q+1}}/_{v_c}{u_1u_2}/\cdots /_{u_pu_{p+1}} \\
		&& \text{ ($v_c$ is the new vertex after pooling $v_{q+1}u_1$ )}\\
		&=&G/_{v_1v_2}/_{v_2v_3}/ \cdots/_{v_qv_{q+1}}/_{u_1v_{q+1}}{u_1u_2}/\cdots /_{u_pu_{p+1}} \\
		&& \text{ (by Proposition \ref{prop} and continue the above steps)}\\
		&\simeq& G/_{u_1u_2}/\cdots /_{u_pu_{p+1}/_{u_{p+1}v_1}/_ {v_1v_2}/_{v_2v_3}/ \cdots/_{v_qv_{q+1}}} \\
		&\simeq& G/_{\mu_4}
	\end{eqnarray*}
\end{proof}
\begin{definition}\label{vset} 
	A  FGP set $\zeta'$ in a f-graph $G: (\sigma, \mu)$ is a set of vertices  $v' \subseteq \sigma^*$, such that spanning fuzzy forest induced by $\zeta'$. A   FGP $C$ of $G$ is a f- graph such that there exists a  FGP set $\zeta'$ with $C = G/_{\zeta'}$.
\end{definition}
\noindent After the previous observations, we develop another view on fuzzy vertex pooling a set $\zeta'$ of vertices. The
graph $G/_{\zeta'}$ is composed of connected components $Q_1, Q_2, \cdots Q_k$. 
\begin{theorem} 
	Let be given a  f-graph $G: (\sigma, \mu)$, $v \in \sigma^*$ and $e \in \mu^*$. Then,
	\begin{eqnarray*}
		d_{s_{G}}(v) \geq d_{s_{G/e}} (v) 		 
	\end{eqnarray*}	
	\begin{proof}
		We prove this by considering distinct cases on $e = pq$.
		\begin{description}
			\item [Case I:] $v$  is not an end vertex of  $e$ and $v\in N[p]\vee N[q]$. 
			\begin{eqnarray*}
				d_{s_{G}}(v) &=&\sum_{w\in N_s(v)}\mu(wv)\\
				&=& \sum_{w\in N_{s_{G}}(v)-\{p,q\}}\mu(wv) + \mu(pv)+\mu(qv)\\
				&\geq& \sum_{w\in N_{s_{G/e}}(v)\}}\mu(wv) \\
				&\geq& d_{s_{G/e}} (v) 		 
			\end{eqnarray*}	
			
			\item [Case II:]  $v$  is not an end vertex of  $e$ and $v\not\in N[p]\vee N[q]$. Then, clearly	$d_{s_{G}}(v) = d_{s_{G/e}} (v) $
			\item [Case III:]  $v$  is  an end vertex of  $e$. Say $v=q$
			\begin{eqnarray*}
				d_{s_{G}}(v) &=&\sum_{w\in N_s(v)}\mu(wv)\\
				&=& \sum_{w\in N_s(v)-\{p\}}\mu(wv) + \mu(pq)\\
				&\geq& \sum_{w\in N_{s_{G/e}}(v)}\mu(wv) \\
				&\geq& d_{s_{G/e}} (v) 		 
			\end{eqnarray*}	
		\end{description}
	\end{proof}
\end{theorem}
Now, we will prove that FGP in CFG is an isomorphism and state two theorems of FGP in f-tree and f-cycle. Later we define three types of FGP.
\begin{theorem}\label{cfg1}
	Let be given a $G: (\sigma, \mu)$  with $|\sigma* | = n$, $G/{pq} :\sigma_c, \mu_c)$ is CFG for $e= pq $ $\in \mu^*$
\end{theorem}
\begin{proof}
	Let $\sigma^* = \{p_1,p_2,\cdots, p_n\}$ be the $n$ vertices with, $\sigma(p_1) \leq \sigma(p_2) \leq \cdots \leq \sigma(p_n)$. Choose any edge $p_jp_k$.\medskip
	
	\noindent Then $G/_{p_jp_k}$ is a f-graph with $n-1$ vertices. Name new vertex as $p'$ with $\sigma(p') = \sigma(p_j)\wedge \sigma(p_k)$.
	$$
	\mu(p_ip')= \left\{\begin{array}{ll}
		\sigma(p_1) &  \text{ if } i < \min\{j,k\}\\
		\wedge(\sigma(p_j), \sigma(p_k))& \text{ else }
	\end{array}\right.
	$$
	Here that $G/_{p_jp_k}$ is a CFG ( Figure \ref{cfg} is an illustration).
	\begin{figure}
		\centering
		\includegraphics[width=6cm, height=6cm]{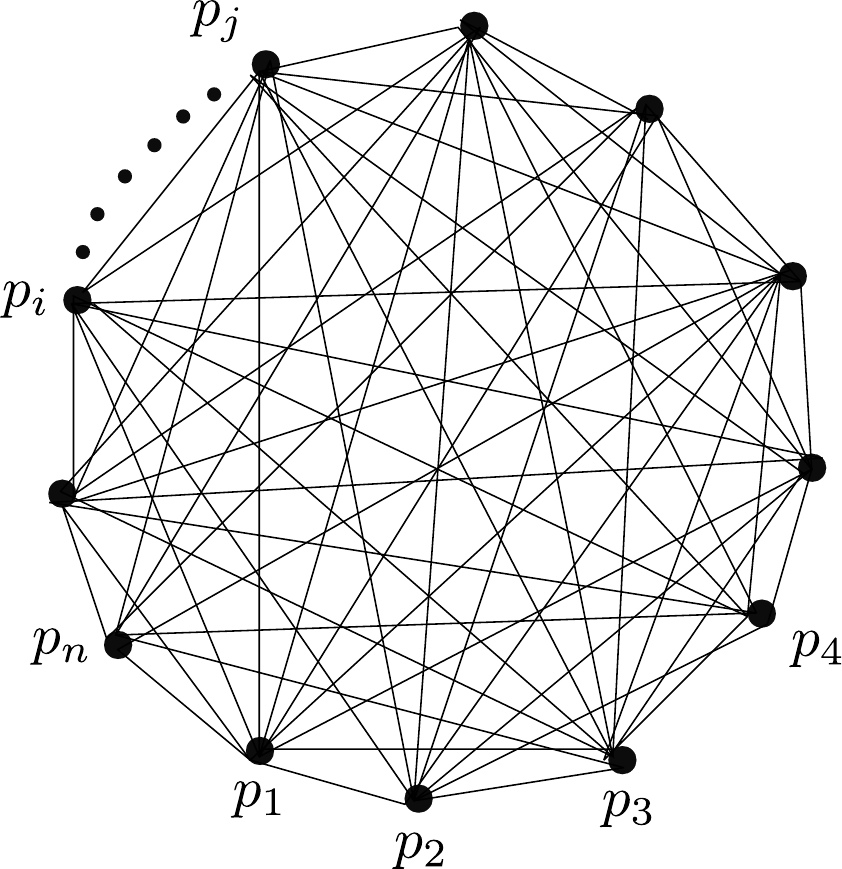}
		\caption{  FGP of a complete f-graph}
		\label{cfg}
	\end{figure}
\end{proof}
\begin{theorem}
	Let be given a CFG $G: (\sigma, \mu)$ and weakest edges $pq, pz \in \mu^*$.  Then $G/{pq} \approx G/{pz}.$
\end{theorem}
\begin{proof}
	Let $G/{pq} = (\sigma_1,\mu_1) \text{ and } G/{pz} = (\sigma_2,\mu_2)$. 
	Looking at definition \ref{def1} and Theorem \ref{cfg1}, we see that $|\sigma_1^*| = |\sigma_2^*|$. And $G/{pq} \text{ and } G/{pz}$ is a complete f-graph with $|\sigma^* - 1|$ vertices.  Therefore, $G/{pq}$ is isomorphic to $G/{pz}$ .
\end{proof}
%\begin{theorem} 
In a f- tree $G: (\sigma, \mu)$  with $|\sigma* | = n$,
%\begin{enumerate}
%\item[(a)] 
$G/_{pq}$ need not be a f- tree if $pq$ is a $\delta $ edge.
%	\item[(b)]
$G/_{pq}$  is a f- tree if $pq$ is an $\alpha$ edge.
%\end{enumerate}
%\end{theorem}
%\begin{theorem} 
In a f- cycle $G: (\sigma, \mu)$  with $|\sigma* | = n$. Let $e= pq \in \mu^*$ is an $\alpha$ edge.  FGP of $G$ 	satisfies the following relationship:
%\begin{enumerate}
%\item[(a)]
if $p$ or $q$ is a cutvertex then $G/_{pq}$  is a f- cycle else $G/_{pq}$  is a f- tree.
%\end{enumerate}
%\end{theorem}
\noindent Thus, how to implement a pooling remains a conundrum. Next,  we introduce three kinds of FGP.
F- cycle pooling (FCP): Identifying f- cycle in a f-graph and pooling f- cycle into a single vertex.
F- block pooling (FBP): Identifying f- block in a f-graph and pooling f- block into a single vertex and CFG pooling (CFGP): Identifying complete fuzzy sub graph in a f-graph and pooling complete f- subgraph into a single vertex.\medskip

$\Pi$, a criteria for FGP if, for any fuzzy graph $G: (\sigma, \mu)$, all vertex pooling or subgraph pooling of $G$ satisfy $\Pi$. $\Pi$ is nontrivial on connected fuzzy graphs  if it is true for connected fuzzy graphs  and false for  many connected f-graphs. The criteria $\Pi$ is determined by the f-graph if, for any f-graph $G$, $G$ 
satisfies $\Pi$ if and only if its fuzzy subgraphs satisfies $\Pi$. $\Pi$  is determined by all of its components if, for any f - graph $G$, $G$ satisfies $\Pi$ if and only if all connetced components of G satisfy $\Pi$ (Figure \ref{fpcri}). \medskip

\noindent
\hfill 
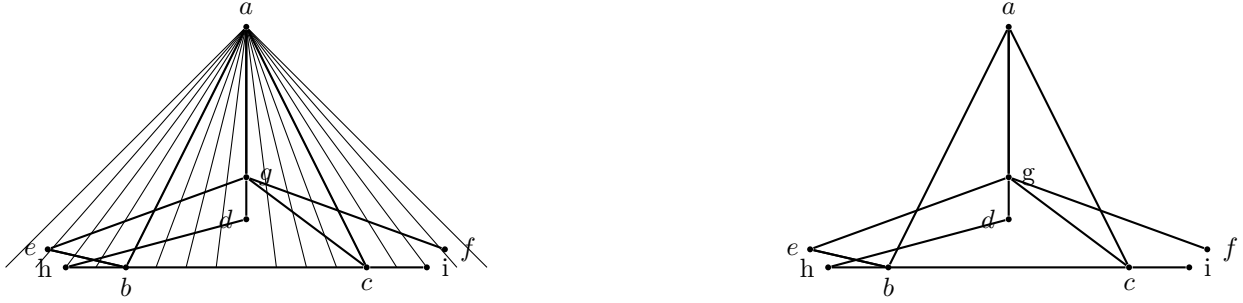
\begin{figure}
    \centering

\begin{tikzpicture}[scale=.8, every node/.style={circle, fill, inner sep=.8pt}, label distance=.8pt]

% Define nodes
\node (a) at (0, 4) [label=above:$a$] {};
\node (b) at (-2, 0) [label=below:$b$] {};
\node (c) at (2, 0) [label=below:$c$] {};
\node (g) at (0, 1.5) [label=right:$g$] {};
\node (d) at (0, 0.8) [label=left:$d$] {};

% Additional nodes at upper vertical position
%\node (1) at (0, 5) [label=above:1] {};
%\node (1p) at (0.3, 5) [label=above:$$] {};

% Additional nodes at horizontal ends
\node (h) at (-3, 0) [label=left:h] {};
\node (e) at (-3.3, 0.3) [label=left:$e$] {};
\node (i) at (3, 0) [label=right:i] {};
\node (f) at (3.3, 0.3) [label=right:$f$] {};

% Draw edges (heavy connections)
\foreach \i in {-2, -1.5, ..., 1.5} {
    \draw (a) -- ($(b) + (\i, 0)$);
    \draw (a) -- ($(c) + (-\i, 0)$);
}

% Attach extra connections for a2 and a3
\draw [thick] (a)--(g);
\draw  [ thick](a) -- (b);
\draw [ thick] (a) -- (c);
\draw [ thick] (a) -- (g) -- (c);
\draw [thick] (g)--(f);
\draw [ thick] (g)--(e);
\draw [ thick] (g)--(d);
\draw [ thick] (d)--(h);

% Draw extra connections for outer nodes
\draw [ thick] (b) -- (e);
\draw [ thick] (b) -- (e);
\draw [ thick] (c) -- (i);
\draw [ thick] (c) -- (h);
%\draw (a1) -- (1);
%\draw (a1) -- (1p);

\end{tikzpicture}
%\end{minipage}%
\hfill
%\begin{minipage}{0.5\textwidth}
% Second Image or TikZ Picture
\begin{tikzpicture}[scale=0.8, every node/.style={circle, fill, inner sep=.8pt}, label distance=.8pt]

% Define nodes
\node (a) at (0, 4) [label=above:$a$] {};
\node (b) at (-2, 0) [label=below:$b$] {};
\node (c) at (2, 0) [label=below:$c$] {};
\node (g) at (0, 1.5) [label=right:g] {};
\node (d) at (0, 0.8) [label=left:$d$] {};

% Additional nodes at upper vertical position
%\node (1) at (0, 5) [label=above:1] {};
%\node (1p) at (0.3, 5) [label=above:$$] {};

% Additional nodes at horizontal ends
\node (h) at (-3, 0) [label=left:h] {};
\node (e) at (-3.3, 0.3) [label=left:$e$] {};
\node (i) at (3, 0) [label=right:i] {};
\node (f) at (3.3, 0.3) [label=right:$f$] {};

% Draw edges (heavy connections)
%\foreach \i in {-2, -1.5, ..., 1.5} {
   % \draw (a) -- ($(b) + (\i, 0)$);
  % \draw (a) -- ($(c) + (-\i, 0)$);
%}

% Attach extra connections for a2 and a3
\draw [thick] (a)--(g);
\draw  [ thick](a) -- (b);
\draw [ thick] (a) -- (c);
\draw [ thick] (a) -- (g) -- (c);
\draw [thick] (g)--(f);
\draw [ thick] (g)--(e);
\draw [ thick] (g)--(d);
\draw [ thick] (d)--(h);

% Draw extra connections for outer nodes
\draw [ thick] (b) -- (e);
\draw [ thick] (b) -- (e);
\draw [ thick] (c) -- (i);
\draw [ thick] (c) -- (h);
%\draw (a1) -- (1);
%\draw (a1) -- (1p);

\end{tikzpicture}
%\end{minipage}
\caption{Fuzzy vertex pooling with Criteria}
\label{fpcri}
\end{figure}

\begin{itemize}
    \item Criteria $\Pi_1$ must be nontrivial on connected f- graphs. A f-graph would be connected if for any pair of vertices, there exists a path where the membership value of edges on the path is above a certain threshold. 
 \item Criteria $\Pi_2$  must be hereditary under FGP. This means that if a f-graph satisfies the property, then after pooling edges (with fuzzy memberships considered), the resulting f-graph should still satisfy the criteria.

 \item Criteria $\Pi_3$ is determined by the undirected f-graphs,  where edges and vertices are described using fuzzy memberships, and the property depends on these fuzzy relations.

 \item Criteria $\Pi_4$ is determined by fuzzy vertex connectivity, where the vertices that meet a specific connectivity threshold.
\end{itemize}

Since almost all f- subgraph graph properties are determined by the f-graph, treats the f-subgraph pooling problem for such a property. Thus, we assume that when we pool vertices and edges, that is, not a directed f- graph.
In this section, we show that the pooling problem is NP-hard if $\Pi$ satisfies: is nontrivial on connected graphs; 
$\Pi$ is hereditary on pooling, $\Pi$ is determined undirected f- graph, $\Pi$ is determined by the strongest paths. 
Note that if a criteria $\Pi$ satisfies above conditions, then a complete f- graph satisfies $\pi$. 

\begin{theorem}
The FGP problem is NP-hard for the criteria $\Pi$. 
\end{theorem}
\begin{proof}
    Let $G: (\sigma, \mu)$ be an instance of f-graph that to be pooled under $\Pi$, an undirected graph with $|\sigma*|$ vertices. Let $G_1: (\sigma_1, \mu_1)$ be the f-graph obtained from $G$ by FGP with criteria $\Pi$ (edges and vertices are pooled using definition \ref{def1}). Let $\mu_1*$ be the set of edges formed. Then $G_1$ has a vertex set $\sigma_1*$ where of $|\sigma_1*| = |\sigma*| - k$, where $k$ is the number of vertices in f- subgraph follows $\Phi$.  Let $H$ be a connected f-subgraph with the minimum number of vertices that not satisfying criteria $\Pi$. Let $pq$ be an edge of $H$. The resulting graph  has properties consistent with 
$\pi$, and the transformation can be done in polynomial time.
\end{proof}

\section{Fuzzy graph informed neural networks}
The neural network designed for this study consists of an input layer of 2 neurons, two hidden
layers of 8 neurons each and the output layer has 2 neurons, Figure \ref{fppp}. The goal of this network is to
determine if the two input values belong to class 0 (their sum is less than 1) or class 1 (their sum
is greater than 1) and generate a decision boundary. Inputs 1 and 2 are floating point (decimal) numbers less than 1. Output 1 and 2 represent classes 0
and 1 respectively. The neuron in the next layer is activated using the sigmoid function on the
weighted sum of all the neurons in the previous layer ( $\sigma(x) = 1/(1+e^x)$).\medskip

\begin{center}
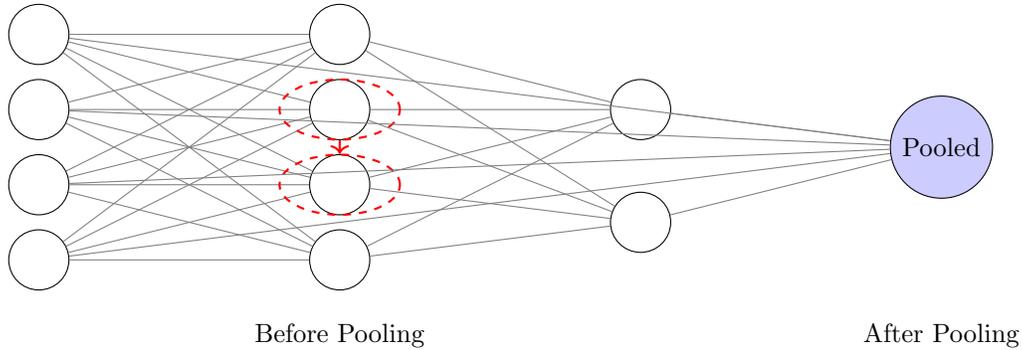
\begin{figure}\centering
\begin{tikzpicture}[x=2cm, y=1cm, every node/.style={circle,draw,minimum size=0.8cm}]

% Input Layer
\node (I1) at (0,2) {};
\node (I2) at (0,1) {};
\node (I3) at (0,0) {};
\node (I4) at (0,-1) {};

% Hidden Layer (Before Pooling)
\node (H1) at (2,2) {};
\node (H2) at (2,1) {};
\node (H3) at (2,0) {};
\node (H4) at (2,-1) {};

% Output Layer
\node (O1) at (4,1) {};
\node (O2) at (4,-0.5) {};

% Connections (Initial)
\foreach \i in {1,2,3,4} {
    \foreach \h in {1,2,3,4} {
        \draw[gray] (I\i) -- (H\h);
    }
}

\foreach \h in {1,2,3,4} {
    \foreach \o in {1,2} {
        \draw[gray] (H\h) -- (O\o);
    }
}

% Pooling Step (Highlight Similar Neurons)
\draw[red,thick,dashed] (H2) circle [radius=0.4];
\draw[red,thick,dashed] (H3) circle [radius=0.4];
\draw[->,red,thick] (H2) -- (H3);

% Merged Neuron
\node[fill=blue!20] (H5) at (6,0.5) {Pooled};
\foreach \i in {1,2,3,4} {
    \draw[gray] (I\i) -- (H5);
}
\foreach \o in {1,2} {
    \draw[gray] (H5) -- (O\o);
}

% Labels
\node[draw=none] at (2,-2) {Before Pooling};
\node[draw=none] at (6,-2) {After Pooling};

\end{tikzpicture}
    	
	\caption{ Fuzzy vertex pooling model}
    \label{fppp}
    \end{figure}
\end{center}

\subsection{Methods and Sources of Data Collection}
The study will comprise of two tests. The dataset of the first set is relatively simple. It comprises
of 100 ordered pairs of floating point numbers. Half of which add up to less than 1 and the other
half, greater than 1. These numbers are randomly generated and their classes were determined.

\begin{table}[h!]
\centering
\begin{tabular}{llll}
data - [ &&&\\
\# Class 0 (Sum $< $ 1) &&&\\
\hline
$[0.1, 0.4]$ & $[0.2, 0.3]$ & $[0.1, 0.5]$ & $[0.3, 0.2]$ \\
$[0.4, 0.1]$ & $[0.3, 0.3]$ & $[0.2, 0.4]$ & $[0.1, 0.6]$ \\
$[0.2, 0.5]$ & $[0.4, 0.3]$ & $[0.3, 0.2]$ & $[0.2, 0.7]$ \\
$[0.3, 0.1]$ & $[0.3, 0.3]$ & $[0.2, 0.6]$ & $[0.2, 0.5]$ \\
$[0.7, 0.4]$ & $[0.8, 0.3]$ & $[0.9, 0.2]$ & $[0.6, 0.4]$ \\
$[0.5, 0.6]$ & $[0.7, 0.3]$ & $[0.8, 0.5]$ & $[0.7, 0.4]$ \\
$[0.9, 0.3]$ & $[0.8, 0.6]$ & $[0.7, 0.5]$ & $[0.6, 0.5]$ \\
$[0.7, 0.4]$ & $[0.9, 0.5]$ & $[0.8, 0.6]$ & $[0.9, 0.4]$ \\
\hline
\end{tabular}
\begin{tabular}{llll}

\# Class 1 (Sum $> $ 1) &&&\\
\hline
$[0.6, 0.5]$ & $[0.7, 0.4]$ & $[0.8, 0.3]$ & $[0.9, 0.2]$ \\
$[0.6, 0.4]$ & $[0.7, 0.3]$ & $[0.8, 0.2]$ & $[0.9, 0.1]$ \\
$[0.6, 0.6]$ & $[0.7, 0.5]$ & $[0.8, 0.4]$ & $[0.9, 0.3]$ \\
$[0.8, 0.7]$ & $[0.9, 0.6]$ & $[0.9, 0.7]$ & $[0.7, 0.8]$ \\
$[0.7, 0.9]$ & $[0.8, 0.8]$ & $[0.9, 0.8]$ & $[0.9, 0.9]$ \\
$[0.9, 0.95]$ & $[0.6, 0.95]$ & $[0.7, 0.85]$ & $[0.8, 0.85]$ \\
$[0.7, 0.65]$ & $[0.6, 0.75]$ & $[0.8, 0.75]$ & $[0.9, 0.75]$ \\
$[0.6, 0.85]$ & $[0.7, 0.85]$ & $[0.9, 0.65]$ & $[0.6, 0.65]$ \\
$[0.7, 0.6]$ & $[0.8, 0.55]$ &  &  \\
\hline
\end{tabular}
\label{tab1}
\caption{Data points table}
\end{table}

\begin{figure}[!htb]
		\centering
		\includegraphics[width=.8\linewidth]{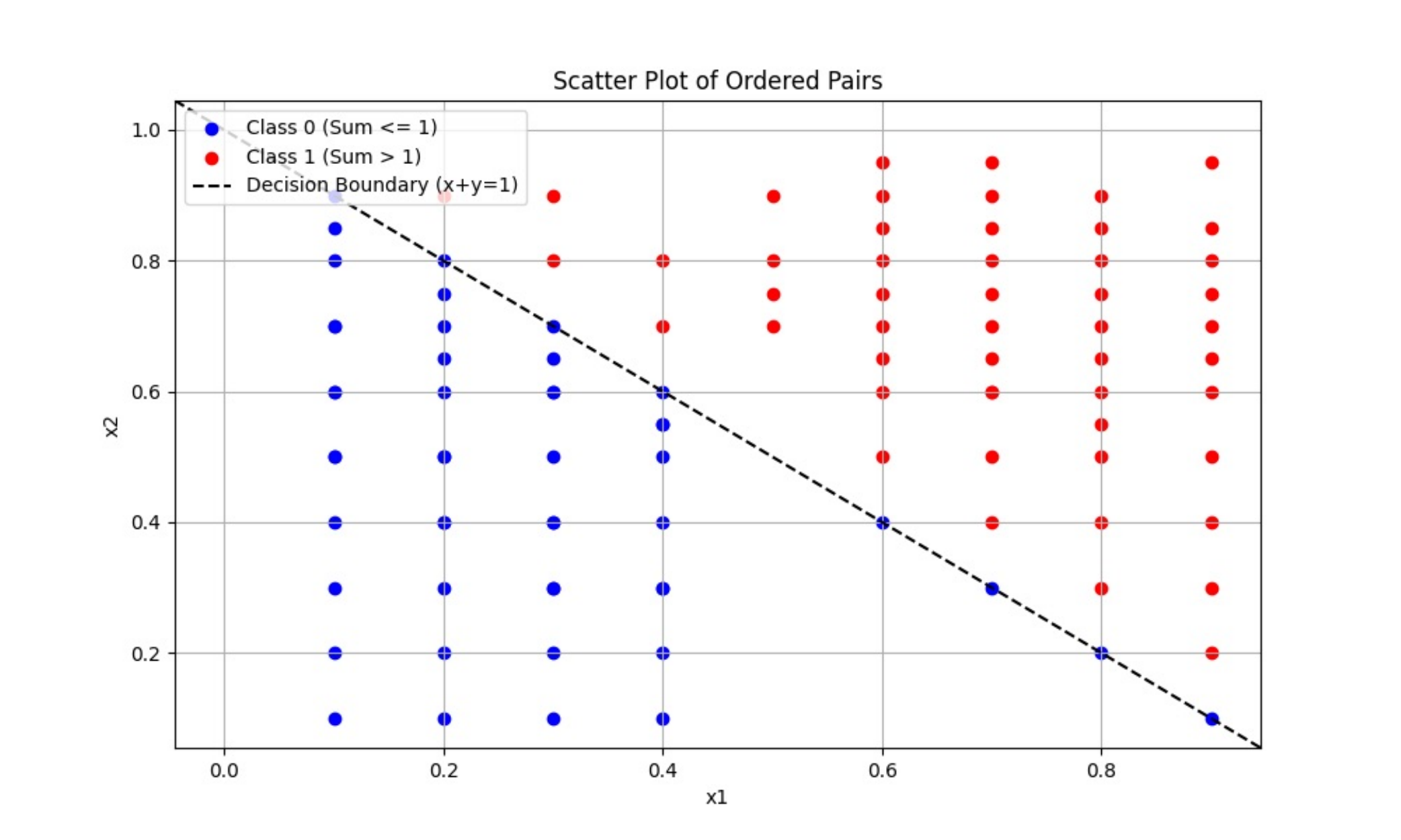}
		\label{f2}
		\caption{ Dataset 1 with 100 floating point ordered pairs and their simple $xy$ scatter plot}\label{Fig:Data2}
	%\end{minipage}
\end{figure}
\begin{table}[h!]
\centering
\begin{tabular}{|l|l|}
\hline
\([0.2510190737913408, 0.6259424942673553]\) & \([0.218839224, 0.405943179971498]\) \\
\([0.2239640999731, 0.36300336]\)           & \([0.7742101924262627, 0.08885469930344578]\) \\
\([0.120308721, 0.7187321846929098]\)       & \([0.27126308384037245, 0.79574132]\) \\
\([0.4456921923751095, 0.1493200825774095]\) & \([0.234270022, 0.5800658340241851]\) \\
\([0.024727105161211158, 0.101667]\)        & \([0.08146331996385, 0.7131972427875436]\) \\
\([0.34589612, 0.943313609963865]\)         & \([0.1051672309900156, 0.00928737]\) \\
\([0.43637059037982, 0.55977085410055]\)    & \([0.957396348, 0.35251081826426]\) \\
\([0.34332206177360427, 0.5430516]\)        & \([0.201810022611168, 0.03063861833530569]\) \\
\([0.46053307, 0.10169015185019025]\)       & \([0.1694911077118097, 0.36330927]\) \\
\([0.55022089924165673, 0.4261096553359686]\) & \([0.3253727, 0.49262867111419906]\) \\
\([0.29022898945996186, 0.643084]\)         & \([0.5802113825489248, 0.01458276552724822]\) \\
\([0.015650422, 0.5117441041653529]\)       & \([0.13984017716325267, 0.808742209]\) \\
\([0.81144393973905, 0.09847477773029108]\) & \([0.4378103, 0.80081139194148]\) \\
\([0.01281428346066545, 0.6418489]\)        & \([0.41026384501725, 0.7622195971674454]\) \\
\([0.1041839, 0.72102484993907]\)           & \([0.45092763427115, 0.3843202]\) \\
\hline
\end{tabular}
	\label{tab2}
\caption{Data points table}
\end{table}

\begin{figure}[!htb]
		\centering
		\includegraphics[width=.8\linewidth]{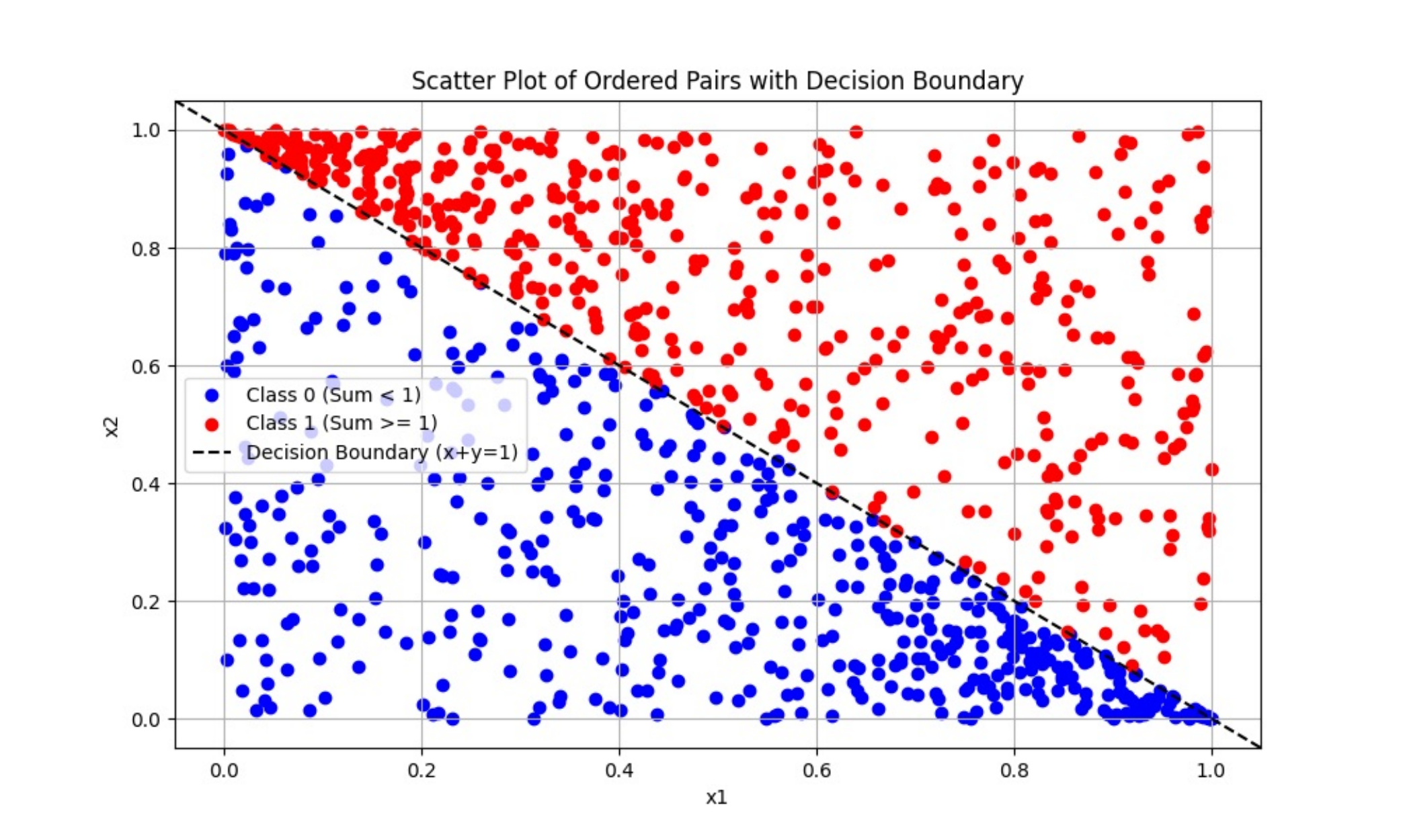}
		\label{f3}
 		\caption{ Dataset 2 with 1000 floating point ordered pairs and their $xy$ scatter plot}
	%\end{minipage}
\end{figure}

\subsection{Fuzzy Graph Model}
The models were coded using numpy library in python. Numpy stands for numeric python and is
useful when working with arrays and matrices. Additional libraries such as matplot lib was also
used to aid in graphing the figures. 
The models use traditional deep learning algorithms that activate the neurons in the next layer
(forward propagation) and then backtrack to minimize any loss or errors (backpropagation). The fuzzy neural networks, both baseline and pooling models, gave best results when the learning rate of the network was set to 0.025. The learning rate determines the step size that the backpropagation algorithm must take to minimize loss. If the step size is too small, the model will take too long to tweak the weights and biases of the model. If the step size is too big, the
backpropagation algorithm would overshoot in its correction and this would exponentially add up
to giving high loss. They were trained with 50,000 epochs, and in all cases, the pooling model was able to minimize
the loss much quicker at the cost of being slightly inaccurate in the testing case for Dataset 2 (1000
floating point numbers, few points are given in Table \ref{tab2}).
pooling are applied in the pooling model every 10,000 epochs. When a pooling is
applied, 2 neurons merge based on a threshold value set to eliminate weak relations. \cite{li2020explain} A pooling changes the internal structure of the hidden layers, which causes the model to
temporarily behave unpredictably. The backpropagation algorithm of a neural network tries to minimize the loss (defined as a
normalized mean square error of the model) after every training iteration (known as an epoch).
Loss is therefore minimized over time.
When both networks were training over Dataset 1 (Table \ref{tab1}), it is observed from Figures 9 and 10 that
the pooling model minimizes the loss much quicker than the baseline model. Scatter plot of the Dataset 1 and 2 is given in Figure 7 and 8.
   \begin{center}
   
    \begin{figure}
     \centering
              \includegraphics[width=10cm, height=5cm]{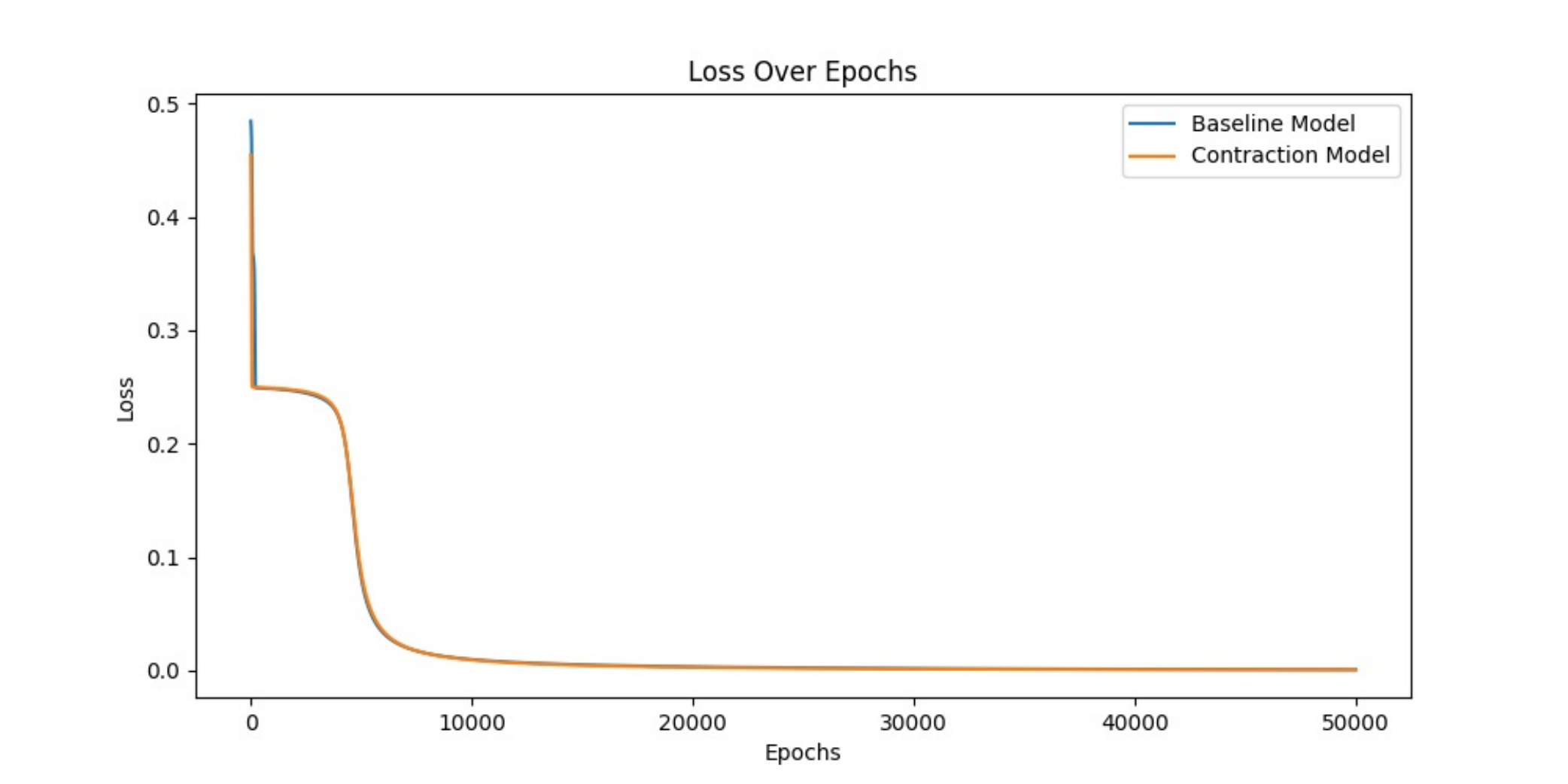}
              	\label{f5}
              \caption{Plot for Loss over 50,000 Epochs}
          \end{figure}

          \begin{figure}%[H]
           \centering
              \includegraphics[width=10cm, height=5cm]{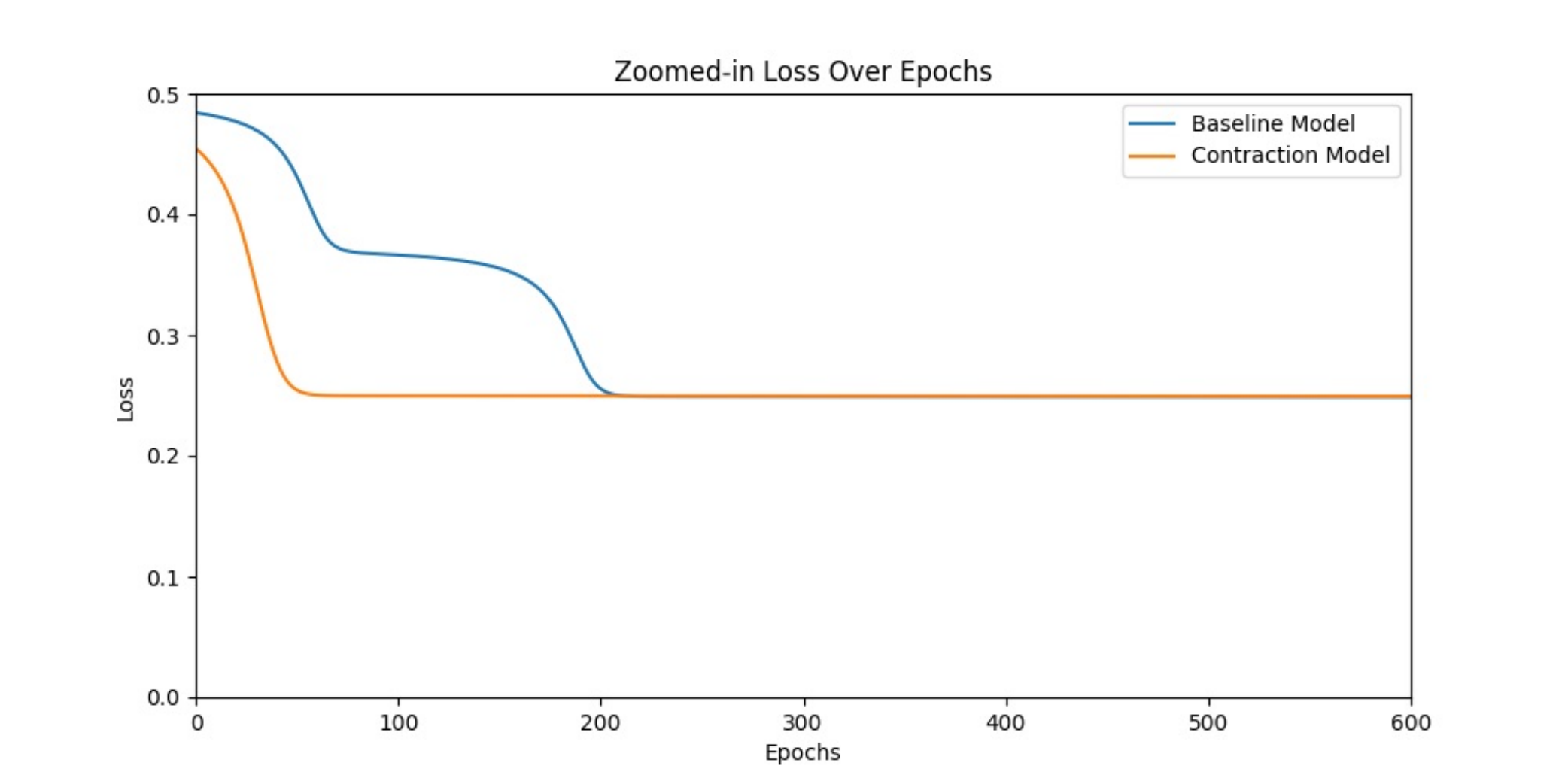}
              	\label{f6}
              \caption{Plot for Loss over the first 600 Epochs}
            \end{figure}
\end{center}
    
\begin{figure}%[H]
%\begin{minipage}[t]{0.45\textwidth}
  \includegraphics[width=8cm, height=6cm]{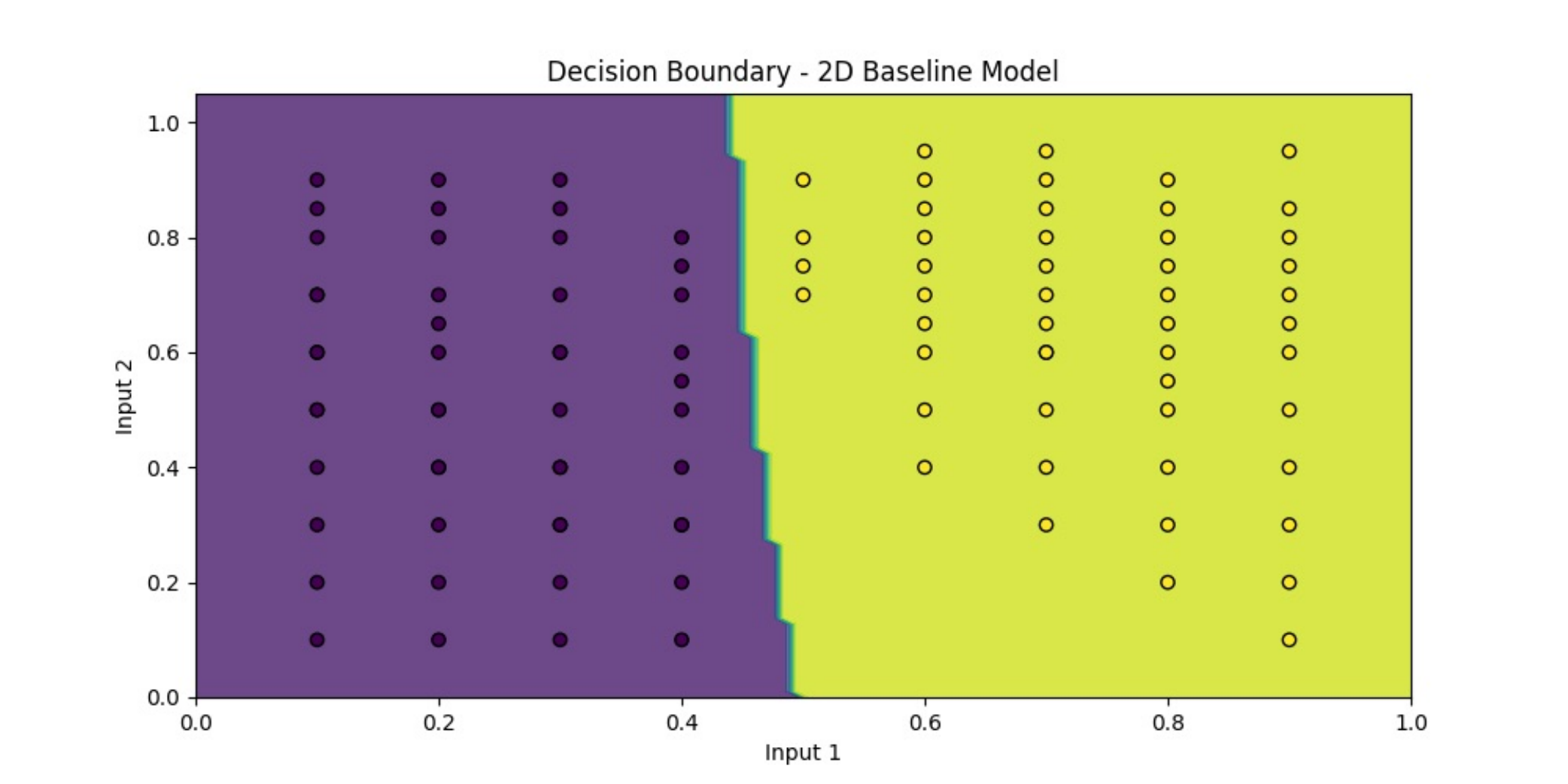}
		\label{f7}
\hfill
  \includegraphics[width=8cm, height=6cm]{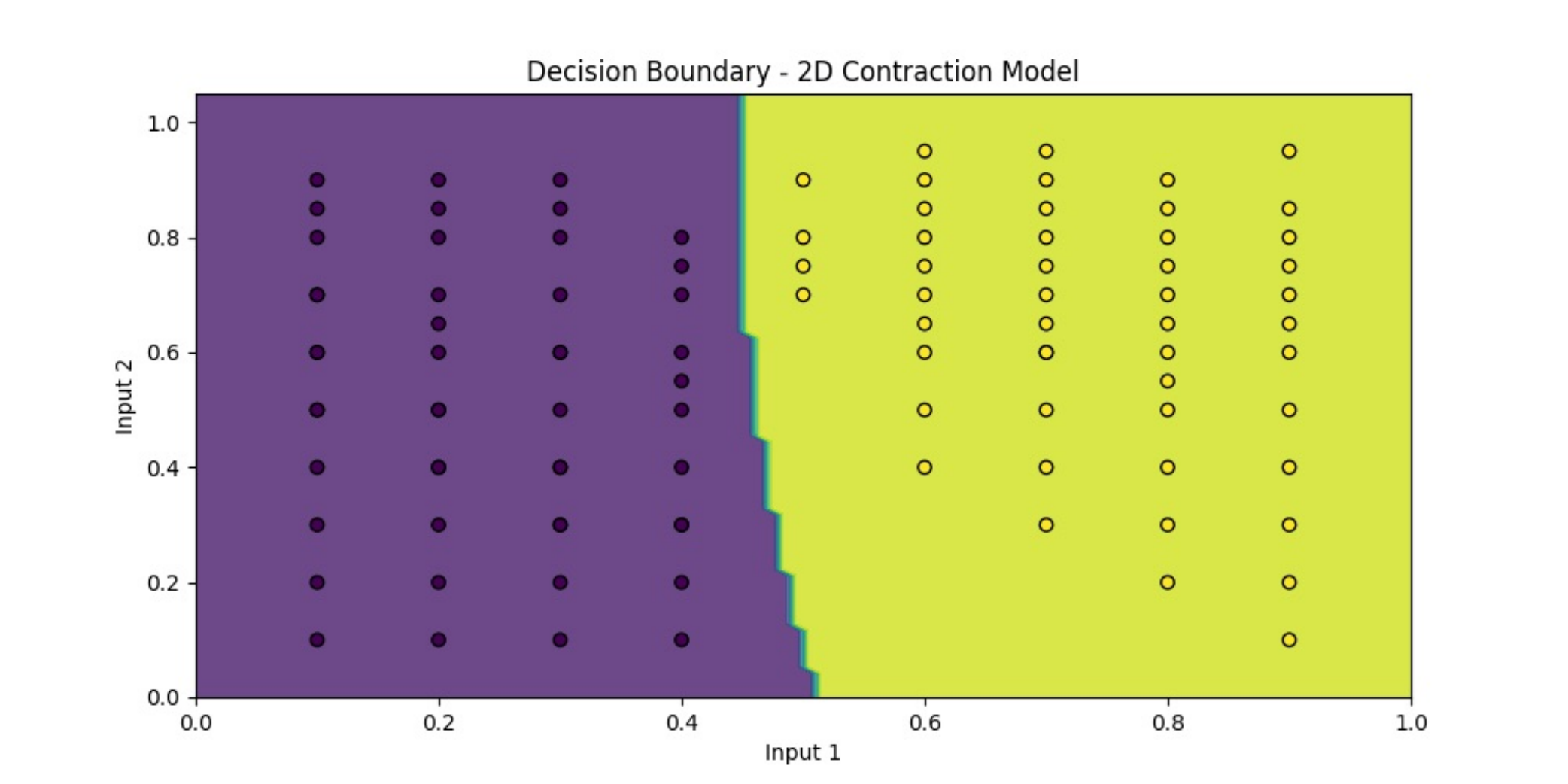}
		\label{f8}
	\caption{  Decision boundary of baseline
mode and contraction model}
    \end{figure}

\begin{figure}%[H]
%\begin{minipage}[t]{0.45\textwidth}
  \includegraphics[width=8cm, height=7cm]{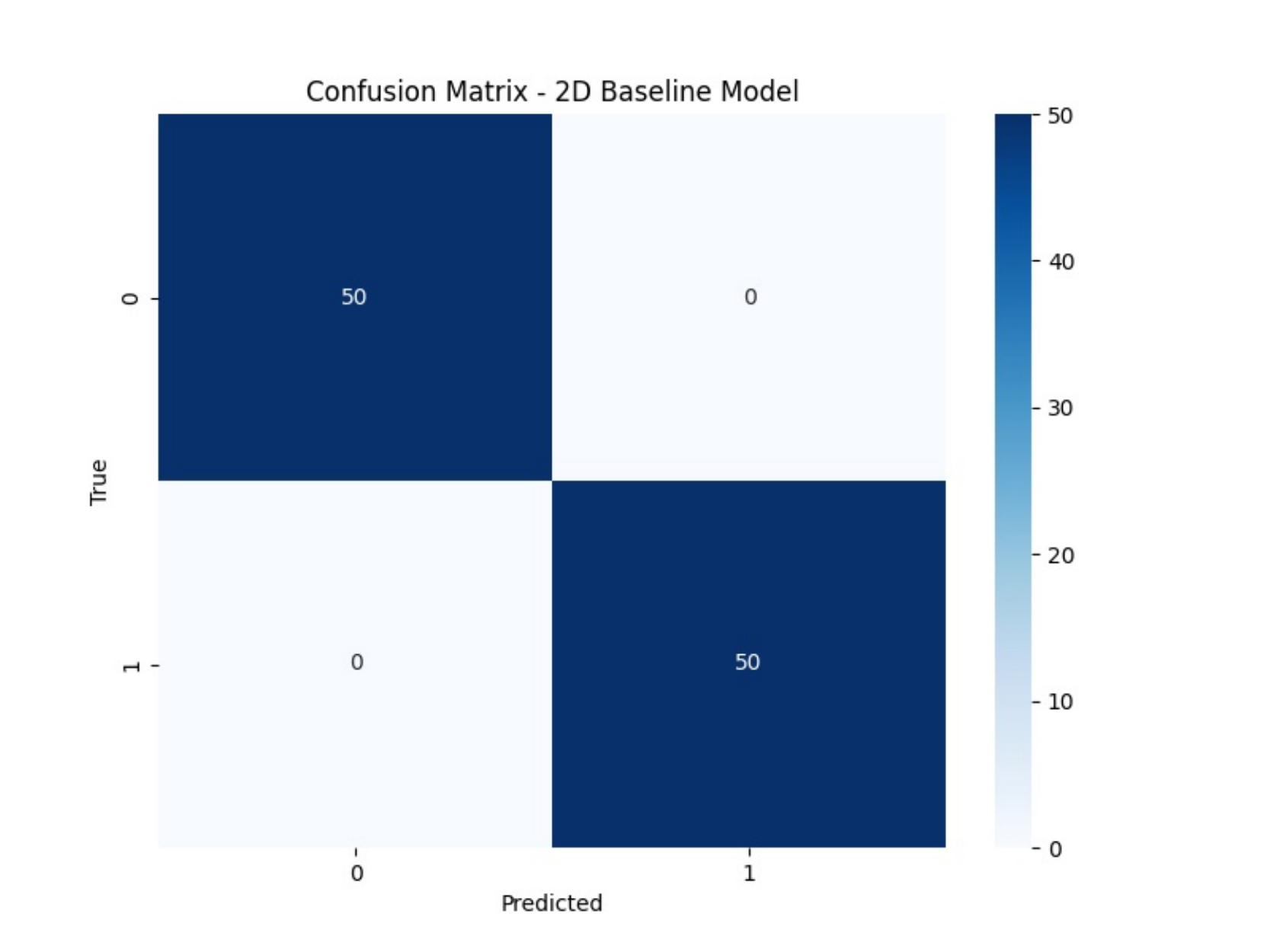}
		\label{f9}
\hfill
  \includegraphics[width=8cm, height=7cm]{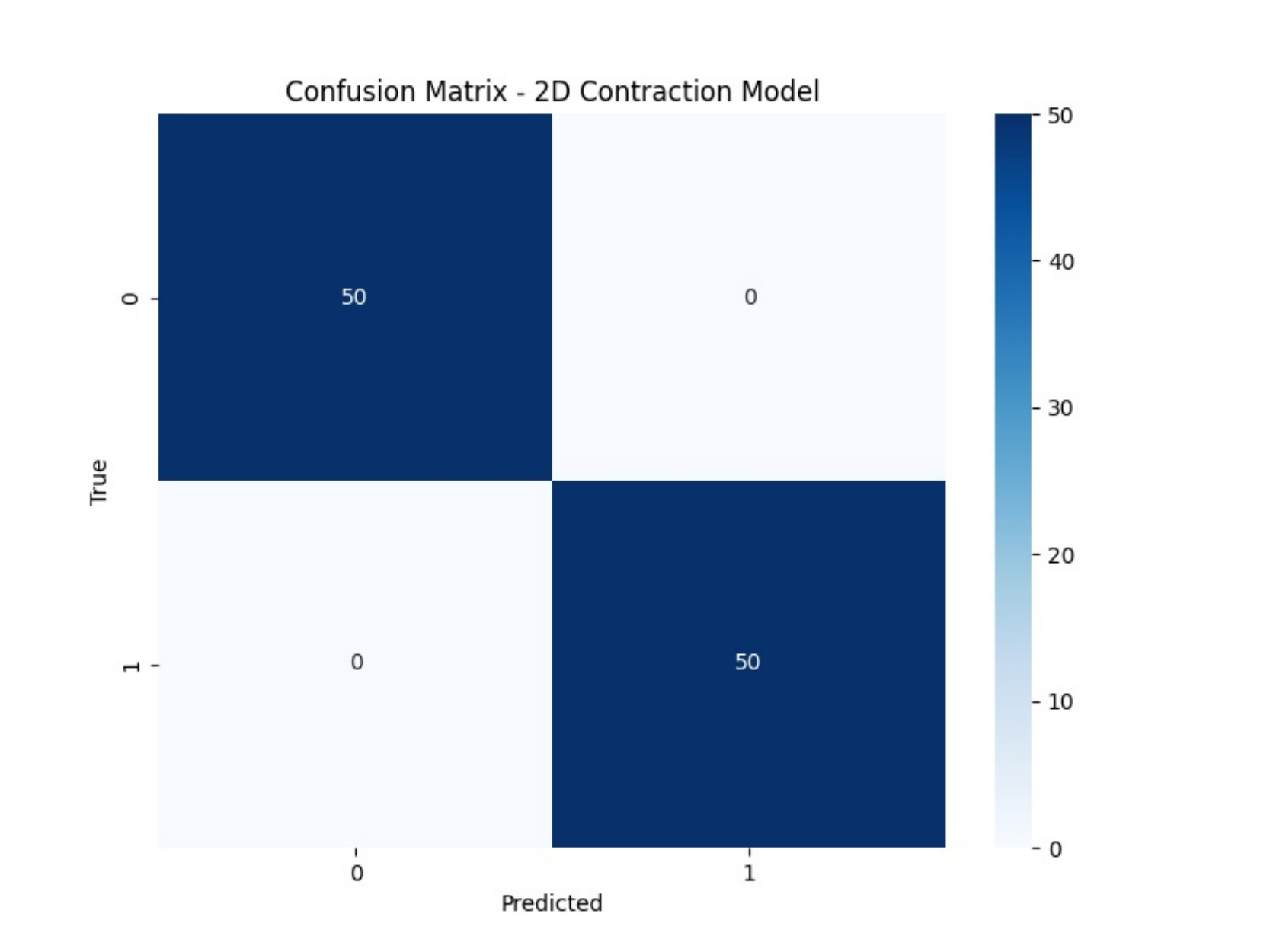}
		\label{f10}
	\caption{  Confusion matrix of baseline model and contraction model}
    \end{figure}

When both networks were training over Dataset 2, it is observed from Figures 13 and 14 that
the pooling model again is much quicker at minimizing loss. It is also important to note that as
pooling are applied every 10,000 epochs, there are spikes in the pooling model’s graph.
This is because a pooling changes the internal structure of the hidden layers, which results in
errors surfacing temporarily. However, this too is quickly minimized.\medskip

A decision boundary is the model’s understanding of where the classes tend to separate in the real valued
vector space of input values. This is commonly used with support vector machines where
multi-dimensional inputs are to be classified with decision boundaries.
A confusion matrix showcases how well a model performs by checking how many predicted values
were actually correct. The matrix’s leading diagonal represent how many predicted true's and false's
were actually correct. Ideal Confusion matrices have their leading diagonal high and their reverse
diagonal zero.
   
   \begin{center}
   
    \begin{figure}
     \centering
              \includegraphics[width=10cm, height=5cm]{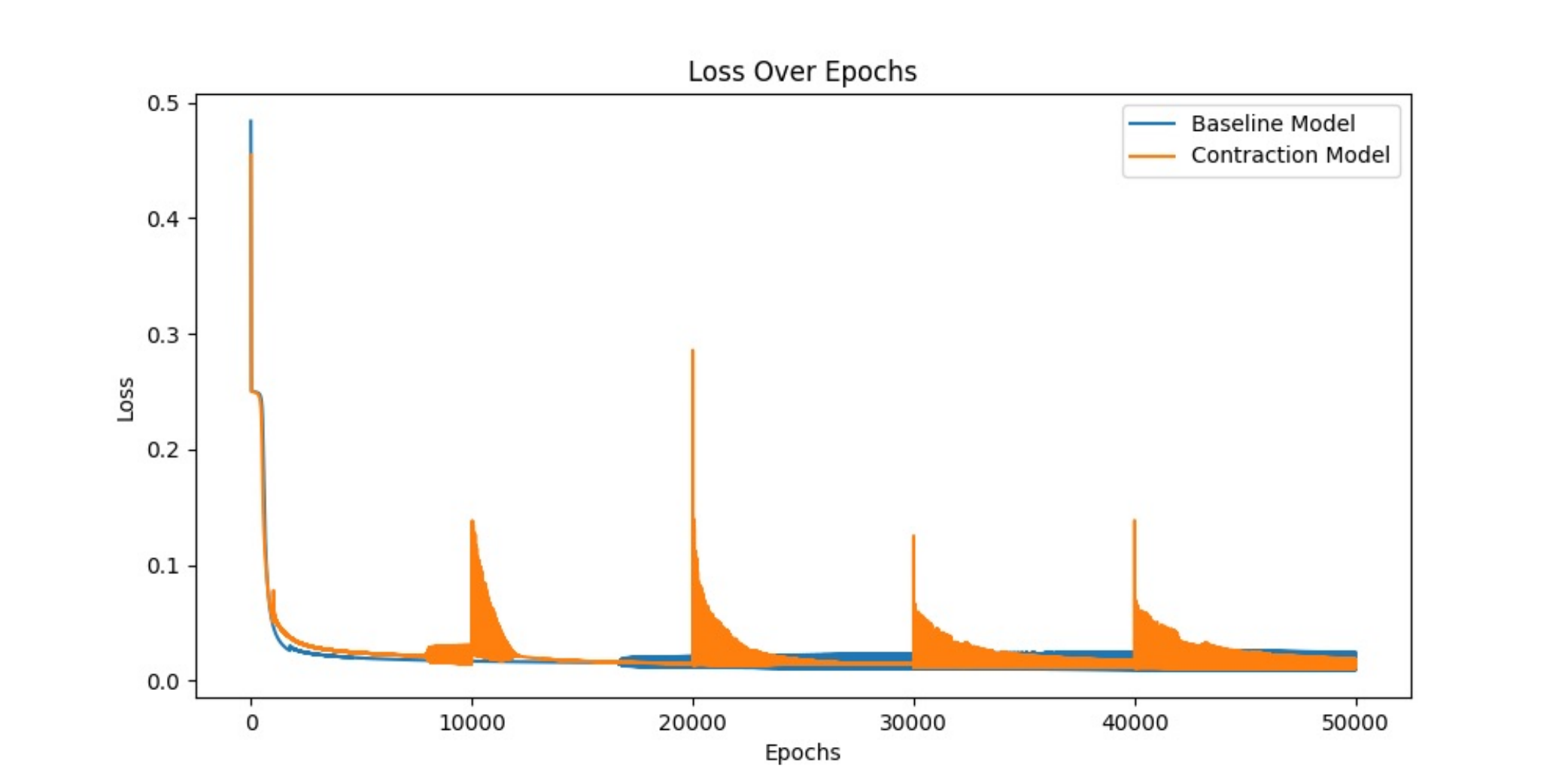}
              	\label{f11}
              \caption{Plot for Loss over 50,000 Epochs}
          \end{figure}

          \begin{figure}%[H]
           \centering
              \includegraphics[width=10cm, height=5cm]{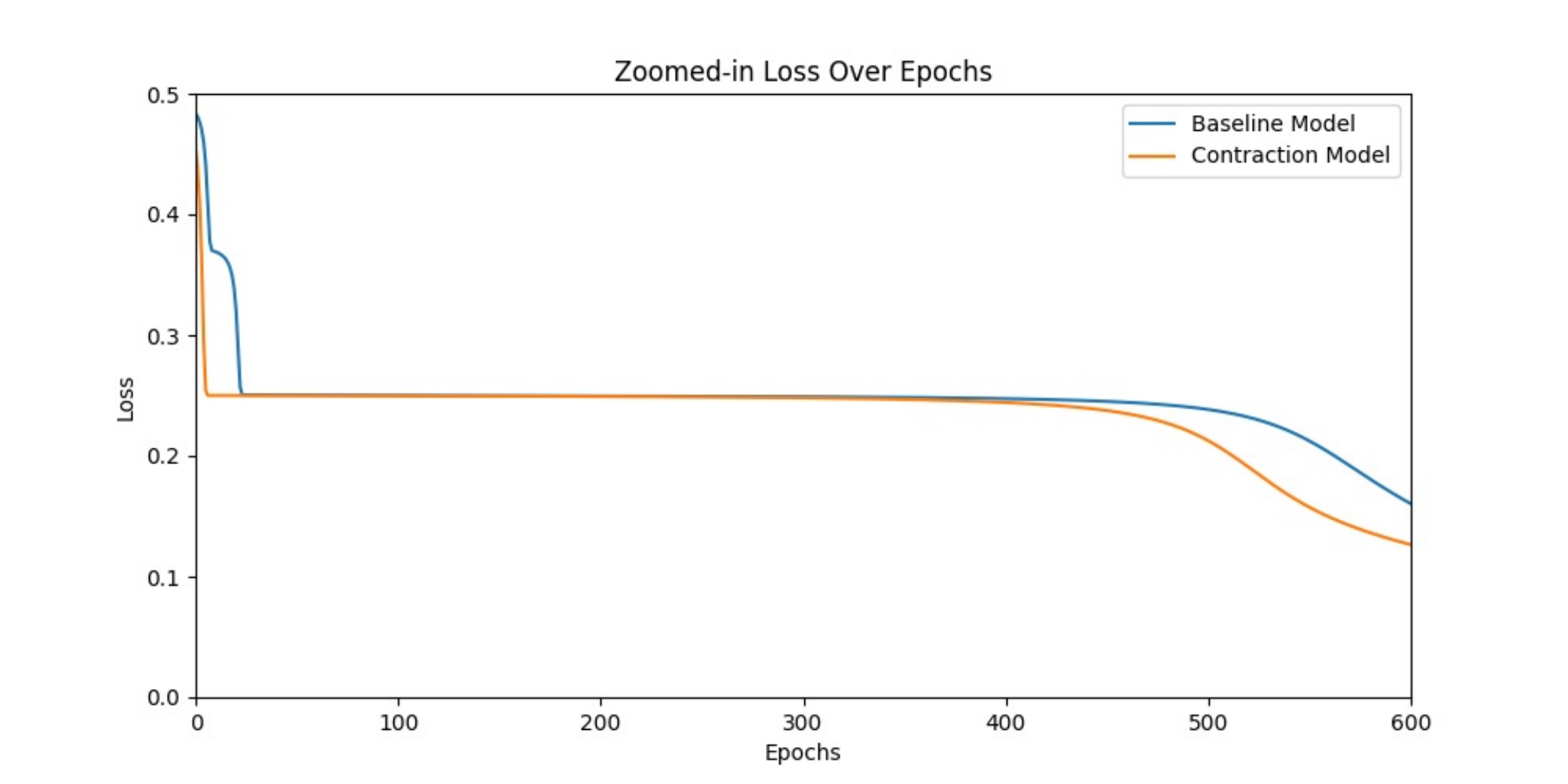}
              	\label{f12}
              \caption{Plot for Loss over the first 600 Epochs}
            \end{figure}
\end{center}
     
  \begin{figure}
  \includegraphics[width=8cm, height=6cm]{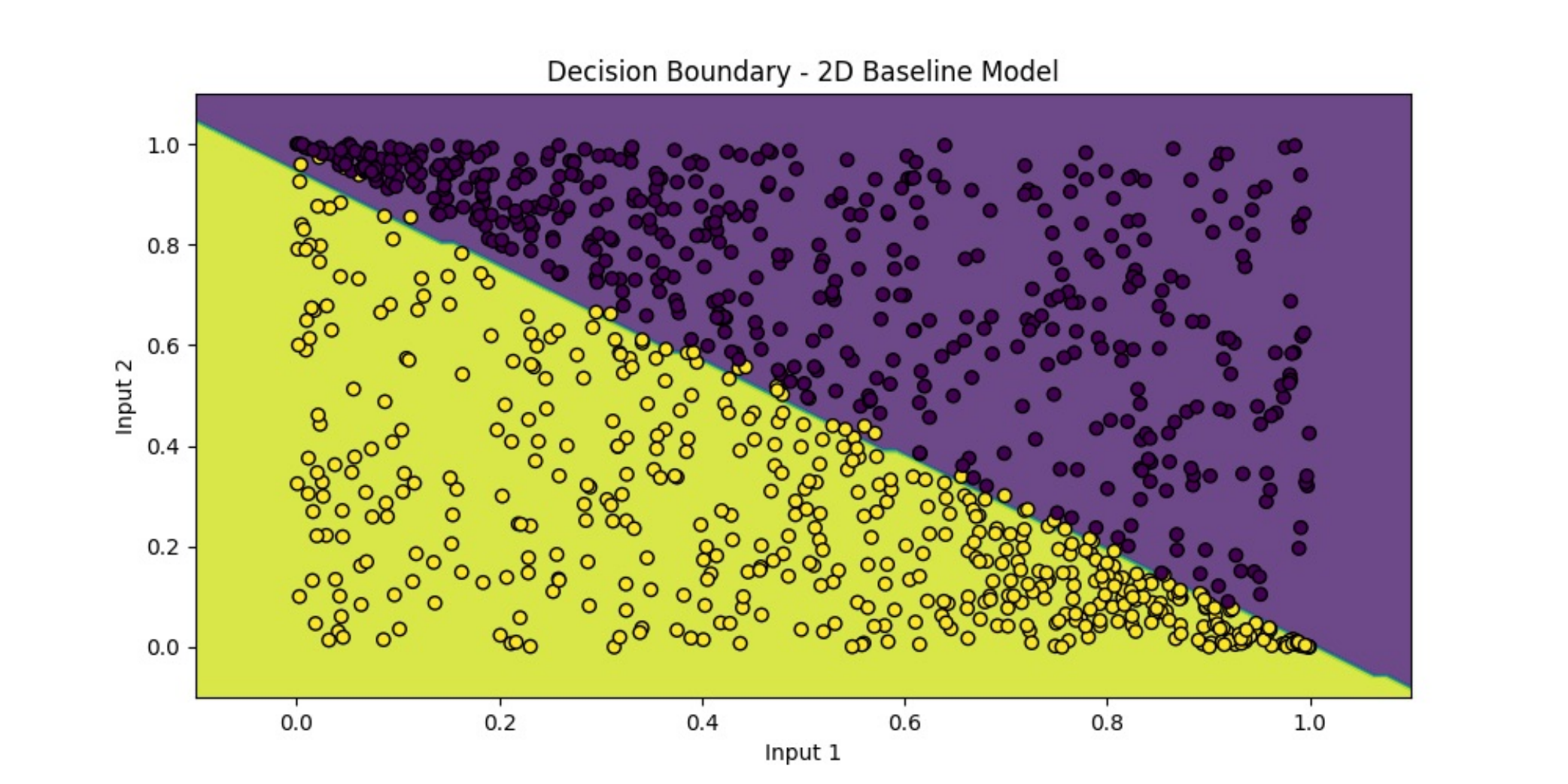}
		\label{f13}
\hfill
  \includegraphics[width=8cm, height=6cm]{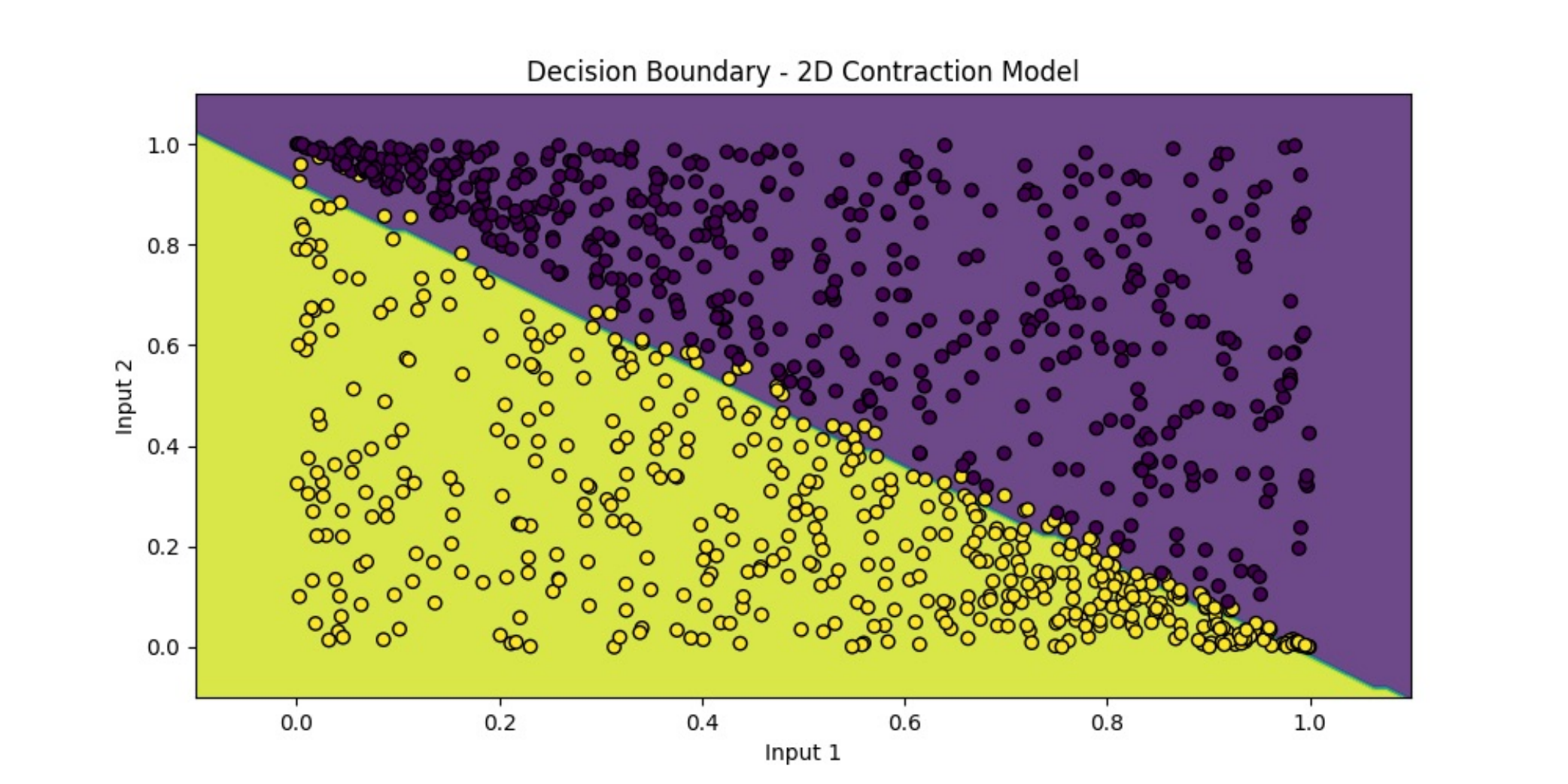}
		\label{f14}
	\caption{ Decision boundary of baseline
mode and contraction model}
    \end{figure}
    
\begin{figure}
  \includegraphics[width=8cm, height=7cm]{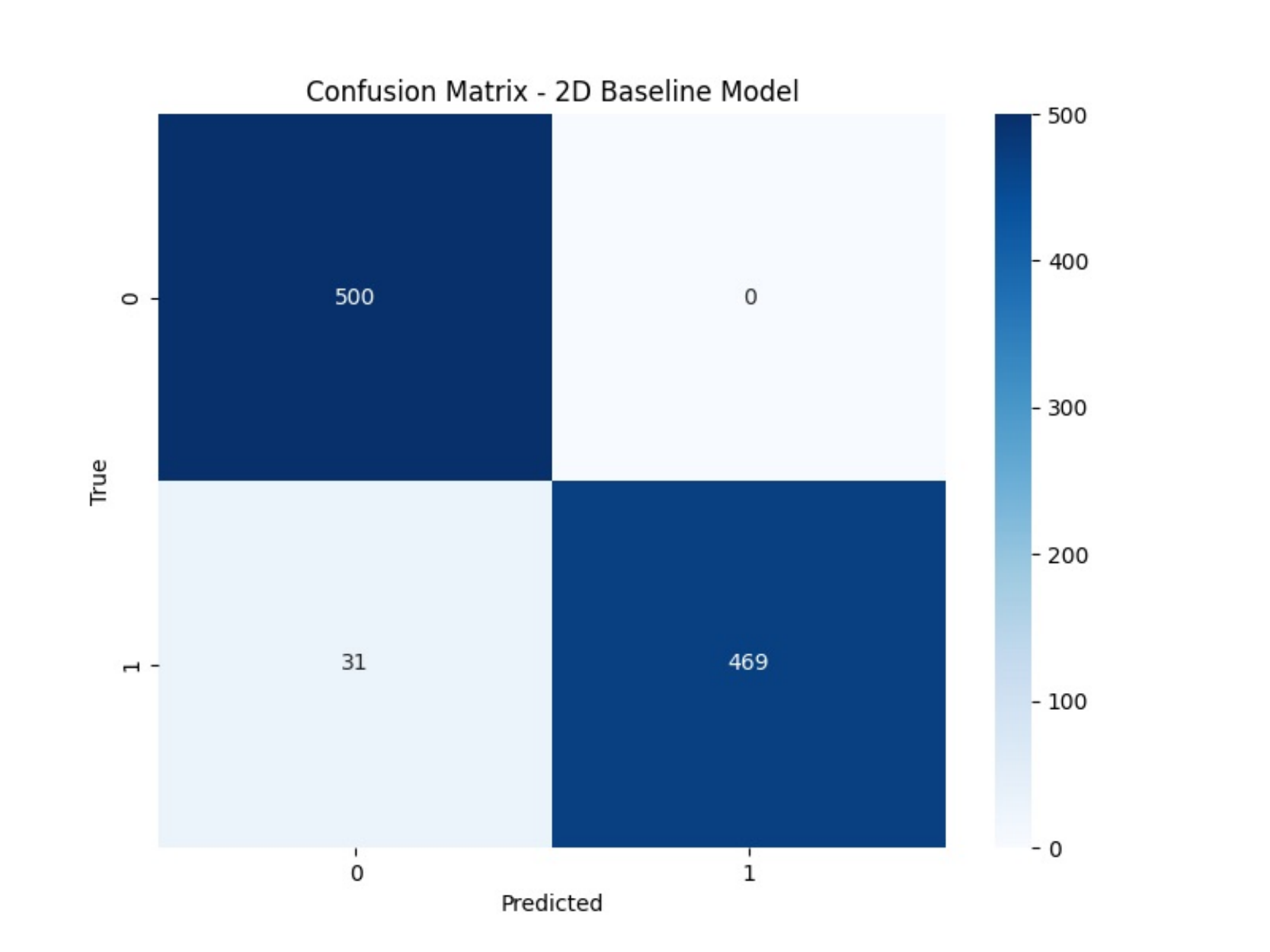}
		\label{f15}
\hfill
  \includegraphics[width=8cm, height=7cm]{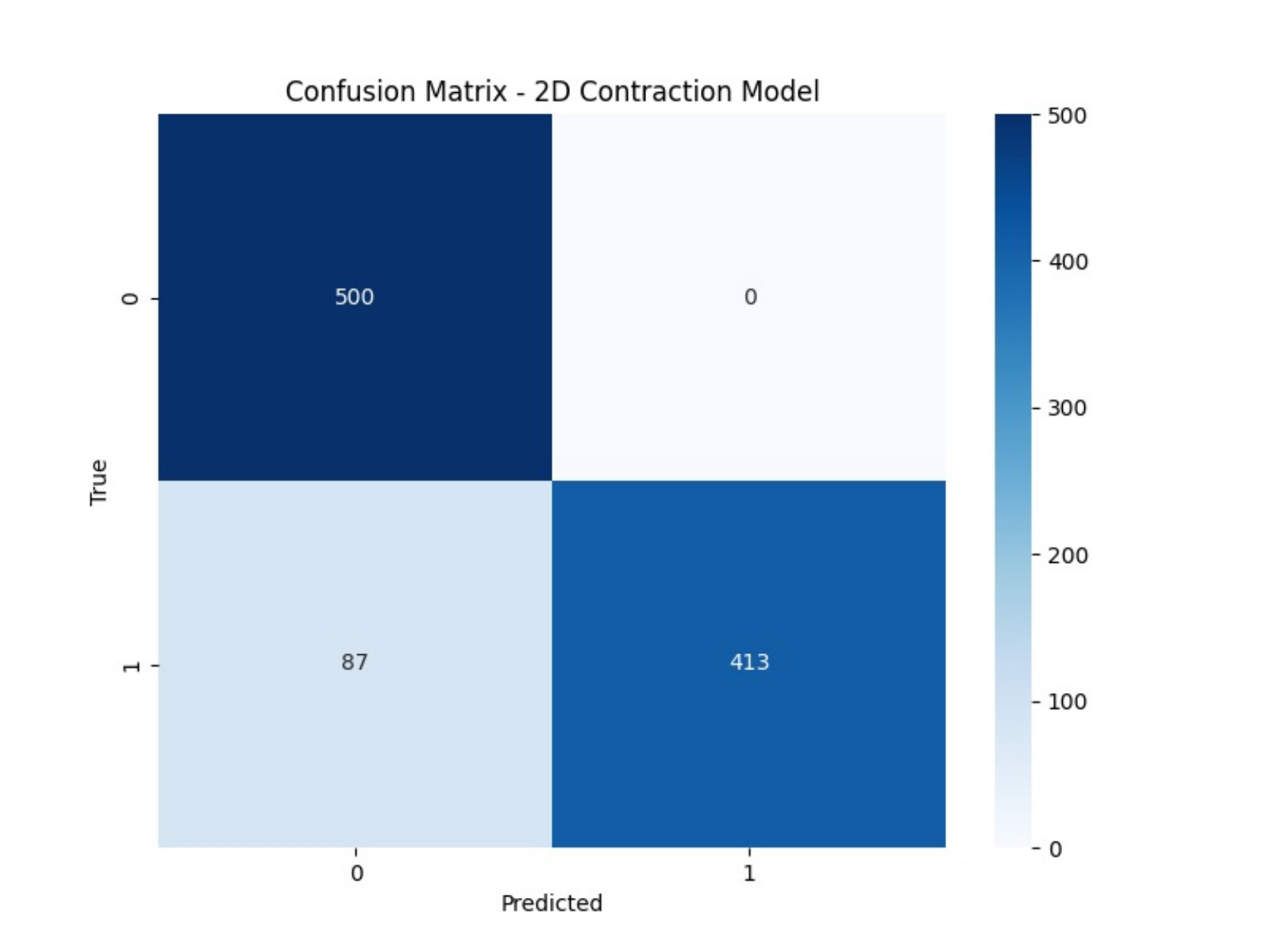}
		\label{f16}
	\caption{  Confusion matrix of baseline model and contraction model}
    \end{figure}

When both networks were training from Dataset 1, both models performed almost identically. This
is evident from the fact that figures 11 are nearly identical and all points are classes
correctly. Therefore, the confusion matrices in figures 12 for both models also generate
a perfect ideal result, where all predicted values were in fact correct.
When both networks were training from Dataset 2, the baseline model was able to classify the
input data more accurately than the pooling model. The pooling model’s decision boundary
does miss a small portion of points [Figures 15]. The confusion matrices for both models
also indicate that the baseline model was more accurate than the pooling model [Figures 16].

\section{FGPNNs Algorithm}
The Feature-Guided Pooling Neural Network (FGPNN) algorithm optimizes neural networks by dynamically merging redundant neurons during training. At specified pooling intervals, the algorithm computes activations of neurons and evaluates pairwise cosine similarity. Neuron pairs with similarity exceeding a defined threshold are merged, reducing over-parameterization and enhancing model efficiency. The merging process creates a new neuron with averaged weights from the selected pair, effectively streamlining the network architecture. This approach decreases model complexity, accelerates inference, and mitigates overfitting by preventing redundant feature learning. FGPNN is particularly advantageous for resource-constrained environments, offering a balance between model size and performance.\medskip

\begin{algorithm}
\caption{Algorithm: FGPNN with Fuzzy Pooling}\label{alg:cap}
\begin{algorithmic}[1]
\Require Neural Network NN, Threshold $\tau$, Pooling Interval $p$, Epochs $max\_epochs$
\Ensure Optimized Neural Network NN
\State Initialize Neural Network NN
\For{epoch = 1 to $max\_epochs$}
    \State Train NN for one epoch using backpropagation
    \If{epoch \% $p == 0$} \Comment{Pooling interval reached}
        \State Compute activations $\mathcal{A} = \{a_1, a_2, \dots, a_N\}$
        \For{each pair of neurons $(i, j)$}
            \State Compute $S_{ij} = \text{cosine\_similarity}(a_i, a_j)$
        \EndFor
        \State Identify pairs $(i, j)$ where $S_{ij} > \tau$
        \For{each identified pair $(i, j)$}
            \State Apply fuzzy pooling:
            \[
            v_c \gets \text{merge}(a_i, a_j)
            \]
            \State Update vertex attributes:
            \[
            \sigma(v_c) \gets \sigma(a_i) \wedge \sigma(a_j)
            \]
            \State Update edge weights for all $u$ connected to $a_i$ or $a_j$:
            \[
            \mu(u, v_c) \gets \mu(u, a_i) \wedge \mu(u, a_j)
            \]
        \EndFor
        \State Update NN structure to reflect pooling
    \EndIf
\EndFor
\Return Optimized Neural Network NN
\end{algorithmic}
\end{algorithm}

For a neural network's hidden layer to achieve structural optimization, its vertices (neurons)  must be dynamically constrained. Let us consider a complete neural network optimization problem of the following general form:
Here, in the illustration,the first layer is input layer, with membership values. This  is the latent solution modeled by the neural network, parameterized by edge membership value  $\mu_i$ and biases $\tau$, $\Pi_i$ represents the vertex membership value (set of neurons) in a hidden layer with particular criteria according to the system, the dataset used for training, are inputs, and   are target outputs.
To optimize the vertices, a pooling operation is applied to reduce weak neurons, constrained by a similarity measure and threshold.
Two types of Vertex Pooling Problems are
forward Pooling Problem, and backward pooling problem.
Given a fixed neural network structure and fixed pooling parameters (e.g., threshold $\tau$, pooling interval p), minimize the loss function  $L(\Pi_i, \mu_i, \sigma_c, \mu_c)$ while reducing redundant neurons in the hidden layer. Specifically, the threshold  $\tau$ and pooling rules are predefined, and the objective is to train the network while dynamically applying pooling every p epochs. The network minimizes the loss function using backpropagation, and the pooling operation reduces the neuron count by merging highly similar neurons.\medskip

The algorithm operates by computing the cosine similarity between neuron activations at specified pooling intervals during training. Neuron pairs $(p, q)$ with similarity exceeding a threshold $\tau$ are identified for merging. Instead of directly averaging weights, fuzzy pooling is applied, where neurons are treated as vertices in a fuzzy graph. The merging of neurons corresponds to identifying two vertices in the graph, pooling them into 
a new vertex $v_c$. This process is formalized by modifying the vertex attributes $\sigma$ and edge weights $\mu$ through fuzzy conjunction operations, ensuring that the new vertex inherits the intersecting features of the original neurons.\medskip
\begin{table}[h]
    \centering
    \small
    \begin{tabular}{m{2cm} m{2cm} m{2cm} m{2cm} m{2cm}}
        \toprule
        \textbf{Epoch Number} & \multicolumn{2}{c}{\textbf{Dataset 1 (100 Floating Point Numbers)}} & \multicolumn{2}{c}{\textbf{Dataset 2 (1000 Floating Point Numbers)}} \\
        & \textbf{Baseline Model Loss} & \textbf{Contraction Model Loss} & \textbf{Baseline Model Loss} & \textbf{Contraction Model Loss} \\
        \midrule
         1000  & 0.2484262 & 0.24902839 & 0.04808859 & 0.05095089 \\
        2000  & 0.24644552 & 0.24737148 & 0.02673441 & 0.03423685 \\
        3000  & 0.24172867 & 0.24322081 & 0.0225607 & 0.02799122 \\
        4000  & 0.22235309 & 0.22322397 & 0.02065131 & 0.02534179 \\
        5000  & 0.08141801 & 0.08621347 & 0.01946583 & 0.02372554 \\
        6000  & 0.03307482 & 0.03479826 & 0.01863874 & 0.02258807 \\
        7000  & 0.02065111 & 0.0210532 & 0.01801961 & 0.02172597 \\
        8000  & 0.01508629 & 0.01499302 & 0.01753345 & 0.02105261 \\
        9000  & 0.01191593 & 0.01160779 & 0.01713826 & 0.02905477 \\
        10000 & 0.00985495 & 0.00944181 & 0.01680855 & 0.02134974 \\
        11000 & 0.00839914 & 0.00792686 & 0.01652781 & 0.02203853 \\
        12000 & 0.00731084 & 0.00680104 & 0.01628488 & 0.02133787 \\
        13000 & 0.00646336 & 0.00592841 & 0.01607185 & 0.01910963 \\
        14000 & 0.00578289 & 0.00523131 & 0.01588298 & 0.01774443 \\
        15000 & 0.00522343 & 0.00466176 & 0.01571395 & 0.01686013 \\
        16000 & 0.00475474 & 0.0041883 & 0.01556151 & 0.01622786 \\
        17000 & 0.0043561 & 0.0037892 & 0.01358965 & 0.01574608 \\
        18000 & 0.00401277 & 0.00344893 & 0.01229743 & 0.01536185 \\
        19000 & 0.00371397 & 0.003156 & 0.01200986 & 0.01504486 \\
        20000 & 0.00345161 & 0.00290172 & 0.01161919 & 0.01477655 \\
        21000 & 0.00321949 & 0.00267938 & 0.01139701 & 0.05289028 \\
        22000 & 0.00301276 & 0.0024837 & 0.01120529 & 0.032185 \\
        23000 & 0.00282758 & 0.00231048 & 0.01155281 & 0.01360401 \\
        24000 & 0.00266084 & 0.00215633 & 0.01119383 & 0.01967304 \\
        25000 & 0.00251002 & 0.00201849 & 0.01252019 & 0.01630115 \\
        26000 & 0.00237303 & 0.00189469 & 0.01281253 & 0.01511133 \\
        27000 & 0.00224815 & 0.00178303 & 0.01078383 & 0.01501028 \\
        28000 & 0.0021339 & 0.00168195 & 0.01504412 & 0.01467604 \\
        29000 & 0.00202906 & 0.00159011 & 0.01034442 & 0.01443588 \\
        30000 & 0.00193257 & 0.00150639 & 0.01032878 & 0.01423167 \\
        40000 & 0.00127634 & 0.0009589 & 0.00961296 & 0.01490049 \\
        50000 & 0.00092503 & 0.00068232 & 0.01582948 & 0.01198592 \\
        \bottomrule
    \end{tabular}
    \label{tab4}
     \caption{Comparison of Baseline and Contraction Model Loss for Two Datasets}
\end{table}

%[H]
%\end{minipage}
\begin{sidewaystable}[htpb]
%\begin{table}[h!]
\centering
\begin{tabular}{|p{3cm}|p{5cm}|p{5cm}|p{5cm}|}
\hline
\textbf{Aspect}                  & \textbf{Details}                                                                   & \textbf{Baseline Model}                           & \textbf{Pooling Model}                              \\ \hline
\textbf{Network Architecture}    & 2 input neurons, 2 hidden layers (8 neurons each), 2 output neurons               & Traditional backpropagation                      & Pooling every 10,000 epochs                         \\ \hline
\textbf{Activation Function}     & Sigmoid ($\sigma(x) = \frac{1}{1 + e^{-x}}$)                                       & Used sigmoid in all layers                      & Same sigmoid function used                         \\ \hline
\textbf{Dataset 1}               & 100 ordered pairs (sum $<$ 1: Class 0; sum $>$ 1: Class 1)                         & Performed well, achieved perfect decision boundaries & Performed similarly, classified all points correctly \\ \hline
\textbf{Dataset 2}               & 1000 ordered pairs (more complex data)                                            & More accurate in classification                & Slightly less accurate; missed some points         \\ \hline
\textbf{Loss Minimization Speed} & Loss minimized using backpropagation, learning rate 0.025                          & Slower minimization                             & Faster minimization with temporary spikes after pooling \\ \hline
\textbf{Learning Rate}           & 0.025                                                                             & Adjusted consistently                           & Same, with quicker adjustments after pooling       \\ \hline
\textbf{Epochs}                  & 50,000 training iterations                                                        & Steady performance throughout                  & Spikes after pooling events, then rapid recovery   \\ \hline
\textbf{Decision Boundary}       & Graphical separation of Class 0 and Class 1 in real-valued vector space           & More precise on Dataset 2                      & Missed a few points on Dataset 2                  \\ \hline
\textbf{Confusion Matrix}        & Evaluated accuracy of classification                                              & Ideal leading diagonal for Dataset 1           & Near-perfect for Dataset 1; lower accuracy for Dataset 2 \\ \hline
\textbf{Best Use Case}           & Scenarios requiring high precision on complex datasets                            & Suitable for high-accuracy applications        & Scenarios prioritizing faster training with tolerance for minor errors \\ \hline
\end{tabular}
\caption{Comparison of Baseline and Pooling Neural Network Models}
\label{tab3}
%\end{table}
\end{sidewaystable}

The pooling step reduces the number of neurons while preserving essential network connectivity and information. The updated network, denoted as $G/_{pq}$, reflects the pooled structure, with connections to neighboring neurons recalculated based on the fuzzy intersection of shared edges. This adaptive pooling strategy enhances model interpretability, prevents overfitting by reducing redundancy, and accelerates inference by decreasing the overall parameter count. FGPNN with fuzzy pooling is particularly useful for resource-constrained environments, where efficiency and accuracy must be balanced.

A predefined value that determines whether two neurons are "similar enough" to be merged. Higher results in more aggressive pooling.
The number of epochs after which pooling is applied during training (e.g., every 10,000 epochs).Two neurons are merged   and   are replaced with a new neuron, This reduces the total number of neurons in the hidden layer.

\section{Conclusion}
In this research, we have explored the potential of fuzzy graph pooling within the realm of
neural networks, with a particular focus on fuzzy graph networks. Fuzzy graphs, with their ability
to represent uncertainty and partial membership, present a versatile framework that can effectively
model complex relationships in real-world scenarios.
The reviewed papers provide a foundation for understanding fuzzy graphs, their operations, and
their potential applications in graph neural networks. The principles of fuzzy graph pooling,
as explored in these studies, offer a structured approach to handling uncertainty and optimizing
information flow within GNN architectures. Future research could focus on implementing and
evaluating these concepts in practical GNN frameworks, with the goal of enhancing their
efficiency, interpretability, and robustness in real-world applications.
The pooling model is notable for minimizing loss incredibly quickly and still holds up in
accuracy despite obviously having less hidden layer neurons. The benefits are short lasted
however, as this study shows that if trained for longer, or with a bigger dataset, the pooling
model does tend to struggle making it infeasible in the later stages of training. It is therefore,  from comparison table (Tables  \ref{tab3} and \ref{tab4} ) recommended that pooling be applied in the early stages of training advanced
models in deep learning.
Overall, this research contributes to the ongoing exploration of fuzzy graphs and their integration
into neural network architectures. By bridging these domains, we aim to advance the understanding
and applicability of fuzzy graphs in modeling and analyzing complex system.

\medskip

%\bibliographystyle{plain}        % basic numeric citations
%\bibliographystyle{unsrt}        % numbered in order of appearance
%\bibliographystyle{abbrv}        % abbreviated author names
%\bibliographystyle{alpha}        % author-year with label initials
%\bibliographystyle{apalike}      % APA-like author-year
%\bibliographystyle{siam}         % SIAM style (if installed)
%\bibliographystyle{ieeetr}       % IEEE style
%\bibliographystyle{acm}          % ACM style
%\bibliographystyle{amsplain}     % AMS
%\bibliographystyle{spmpsci}
%\bibliographystyle{apalike}
%\bibliography{Ref}

\end{document}